\documentclass[10pt,journal,compsoc]{IEEEtran}
\usepackage{amsmath,amsfonts}
\usepackage{algorithmic}
\usepackage{algorithm}
\usepackage{array}
\usepackage[caption=false,font=normalsize,labelfont=sf,textfont=sf]{subfig}
\usepackage{textcomp}
\usepackage{stfloats}
\usepackage{url}
\usepackage{verbatim}
\usepackage{graphicx}
\usepackage{cite}
\usepackage{booktabs}
\usepackage{enumitem}
\usepackage{amssymb}
\usepackage{amsmath}
\usepackage{amsthm}
\usepackage{bm}
\usepackage{xfrac}
\usepackage{xcolor}
\usepackage{multirow}
\usepackage{threeparttable}
\usepackage{cleveref}
\usepackage{xfrac}

\newtheorem{theorem}{Theorem}
\newtheorem{lemma}[theorem]{Lemma}

\newtheorem{proposition}[theorem]{Proposition}

\newcommand{\LineComment}[1]{\STATE \(//\) \texttt{#1}}

\def\ie{i.e.}
\def\eg{e.g.}
\def\aka{aka.}

\def\etal{\emph{et al.}}
\DeclareMathOperator*{\argmin}{argmin}

\newcommand{\topelement}[1]{\scriptsize{#1}}

\newcommand{\defineq}{\triangleq}
\def\xneg{\mathop{\breve{\bm{x}}}}
\def\xpos{\mathop{\bm{a}}}
\def\xqtb{\mathop{\tilde{\bm{a}}}}

\def\insx{\bm{x}}

\def\vecw{\bm{w}}
\def\newy{\ddot{y}}
\def\comp{\mathcal{O}}

\def\newDist{\mathbf{D}}
\def\newdist{\mathbf{d}}
\def\newFist{\mathbf{D}_f}
\def\newfist{\mathbf{df}}

\newcommand{\extDist}[1]{\mathbf{D}_{\cdot,\bm{a}}^\text{#1}(S,a_i)}
\newcommand{\extEist}[1]{\mathbf{D}_{\cdot,\bm{a}}^\text{#1}(S)}
\newcommand{\extBist}{\mathbf{D}_\cdot(S_1,\bar{S}_1)}
\newcommand{\extprev}[1]{\mathbf{d}\big( (\breve{\bm{x}},{#1}),(\breve{\bm{x}}',{#1}') \big)}

\def\ppsdisfull{Harmonic Fairness measure via Manifolds}
\def\ppsdisabbr{\emph{HFM}}
\def\ppssubfull{Acceleration sub-procedure}%
\def\ppssubabbr{\emph{AcceleDist}}
\def\ppsalgfull{Approximation of distance between sets for one sensitive attribute with multiple values}
\def\ppsalgabbr{\emph{ApproxDist}}
\def\ppsuperfull{Approximation of extended distance between sets for several sensitive attributes with multiple values}
\def\ppsuperabbr{\emph{ExtendDist}}
\newcommand{\extDistAlt}[1]{\mathbf{D}_\mathbf{a}^\text{#1}(S,a_i)}
\newcommand{\extEistAlt}[1]{\mathbf{D}_\mathbf{a}^\text{#1}(S)}
\newcommand{\extBistAlt}[1]{\mathbf{D}_{#1}(S_1,\bar{S}_1)}

\newcommand{\topequation}{%
  \setlength\abovedisplayskip{2pt}%
  \setlength\belowdisplayskip{2pt}}

\begin{document}

\def\entitle{Approximating Discrimination Within Models When Faced With Several Non-Binary Sensitive Attributes}
\title{\entitle}

\author{Yijun Bian$^*$, Yujie Luo$^*$, and Ping Xu,~\IEEEmembership{Member,~IEEE}%
\thanks{These authors contributed equally, listed in alphabetical order.}%
\thanks{
Correspondence to Yijun Bian and Yujie Luo.}%
\thanks{Y. Bian is with the Department of Computer Science, University of Copenhagen, 2100 Copenhagen, Denmark (e-mail: yibi@di.ku.dk).}%
\thanks{Y. Luo is with the Department of Mathematics, National University of Singapore, Singapore 117543 (e-mail: lyj96@nus.edu.sg).}%
\thanks{P. Xu is with the Department of Electrical and Computer Engineering, The University of Texas Rio Grande Valley, TX 78539, United States (e-mail: ping.t.xu@utrgv.edu).}
}

\markboth{JOURNAL OF \LaTeX{} CLASS FILES,~Vol.~, No.~, August~2021}%
{Shell \MakeLowercase{\textit{et al.}}: A Sample Article Using IEEEtran.cls for IEEE Journals}

\hyphenation{pre-diction}

\maketitle

\begin{abstract} 
Discrimination mitigation within machine learning (ML) models could be complicated because multiple factors may be interwoven hierarchically and historically. 
Yet few existing fairness measures can capture the discrimination level within ML models in the face of multiple sensitive attributes (SAs). 
To bridge this gap, we propose a fairness measure based on distances between sets from a manifold perspective, named as `\emph{\ppsdisfull{}} (\ppsdisabbr{})' with two optional versions, which can deal with a fine-grained discrimination evaluation for several SAs of multiple values. 
Because directly computing \ppsdisabbr{} may be costly, to accelerate its subprocedure---the computation of distances of sets, 
we further propose two approximation algorithms named `\emph{\ppsalgfull{}} (\ppsalgabbr{})' and `\emph{\ppsuperfull{}} (\ppsuperabbr{})' to respectively resolve bias evaluation of one single SA with multiple values and that of several SAs with multiple values. 
Moreover, we provide an algorithmic effectiveness analysis for \ppsalgabbr{} under certain assumptions to explain how well it could work. 
The empirical results demonstrate that our proposed fairness measure \ppsdisabbr{} is valid and approximation algorithms (\ie{} \ppsalgabbr{} and \ppsuperabbr{}) are effective and efficient.\looseness=-1 
\end{abstract}

\begin{IEEEkeywords}
Fairness, machine learning, multi-attribute protection
\end{IEEEkeywords}

\section{Introduction}

\IEEEPARstart{A}{s} techniques of machine learning (ML) and deep learning (DL) are flourishingly developed and ML/DL systems are widely deployed in real life nowadays, 
concerns about the underlying discrimination hidden in these models has grown, particularly in high-stakes domains such as healthcare, recruitment, and jurisdiction \cite{pessach2022review}, where equity for all stakeholders is pivotal to prevent unjust outcomes, akin to a discriminatory Matthew Effect. 
It is important to prevent ML models from perpetuating or exacerbating inappropriate human prejudices not only for model performance but also for societal welfare. 
Effectively addressing and eliminating discrimination usually requires a comprehensive grasp of its occurrence, causes, and mechanisms. 
For instance, a case involving a person changing their gender for lower car insurance rates highlights the complexity of fairness in ML. 

Although the impressive practical advancements of ML and DL thrive on abundant data, their trustworthiness and equity heavily hinge on data quality. 
In fact, one of the primary sources of unfairness identified in the existing literature is biases from the data, possibly collected from various sources such as device measurements and historically biassed human decisions \cite{verma2018fairness}. 
Moreover, the challenge of data imbalance often looms in human-sensitive domains, amplifying concerns of discrimination and bias propagation in ML models. 
As a result, misformed model training would amplify imbalances and biases in data, with wide-reaching societal implications. 
For example, optimising aggregated prediction errors can advantage privileged groups over marginalised ones. 
In addition, missing data such as instances or values may introduce disparities between the dataset and the target population, leading to biassed results as well. 
Therefore, in order to ensure fairness and mitigate biases, it is crucial to properly address data imbalance and prevent ML models from perpetuating or even exacerbating inappropriate human prejudices.

To mitigate bias within ML models, the very first step is to promptly recognise its occurrence. 
However, promptly detecting discrimination fully, truly, and faithfully is not quite easy because of plenty of factors interwoven with each other. 
First, learning algorithms might yield unfair outcomes even with purely clean data due to proxy attributes for sensitive features or biassed algorithmic objectives. 
For instance, the educational background of one person might be a proxy attribute for those born in families with a preference for boys. 
Second, the existence of multiple sensitive attributes (SAs) and their interaction with each other highlights the complexity of bias tackling, like one member from a marginalised group could become one of the majority concerning another factor, or vice versa. 
Third, dynamic changes and historical factors may need to be taken into account, as bias hidden in data, data imbalance, and present decisions may interweave, causing interrelated impacts and vicious circles. 
Despite many fairness measures that have been proposed to facilitate bias mitigation, most of them mainly focus on one single SA or ones with binary values, and few could handle bias appropriately when facing multiple SAs with even multiple values. 
Therefore, it motivates us to investigate a proper tool to deal with bias in such aforementioned scenarios.

In this paper, we investigate the possibility of assessing the discrimination level of ML models in the presence of several SAs with multiple values. 
To this end, we introduce a novel fairness measure from a manifold perspective, named `\emph{\ppsdisfull{} (}\ppsdisabbr\emph{)}', with two optional versions (that is, maximum \ppsdisabbr{} and average \ppsdisabbr). 
However, the direct calculation of \ppsdisabbr{} lies on a core distance between two sets, which might be pretty costly. 
Therefore, we further propose two approximation algorithms that quickly estimate the distance between sets, named as `\emph{\ppsalgfull{} (}\ppsalgabbr\emph{)}' and `\emph{\ppsuperfull{} (}\ppsuperabbr\emph{)}' respectively, 
in order to speed up the calculation, accelerate the bias evaluation, and broaden its practical applicability. 
Furthermore, we also investigate their algorithmic properties under certain reasonable assumptions, in other words, how effective they could be in achieving the approximation goal. 
Our contribution in this work is four-fold:\looseness=-1
\begin{itemize}
\item We propose a fairness measure named \ppsdisabbr{} that could reflect the discrimination level of classifiers even simultaneously facing several SAs with multiple values. Note that \ppsdisabbr{} has two optional versions, of which both are built upon a concept of distances between sets from the manifold perspective. 
\item We propose two approximation algorithms (that is, \ppsalgabbr{} and \ppsuperabbr{}) that accelerate the estimation of distances between sets, to mitigate the disadvantage of costly direct calculation of \ppsdisabbr{}. 
\item We further investigate the algorithmic effectiveness of \ppsalgabbr{} under certain assumptions and provide detailed explanations. 
\item Comprehensive experiments are conducted to demonstrate the effectiveness of the proposed \ppsdisabbr{} and approximation algorithms. 
\end{itemize}

\section{Related Work}

In this section, we firstly introduce existing techniques to enhance fairness and then summarise available metrics to measure fairness for ML models in turn.

\subsection{Techniques to enhance fairness}
Existing mechanisms to mitigate biases and enhance fairness in ML models can typically be divided into three types: pre-processing, in-processing, and post-processing mechanisms, based on when manipulations are applied during model training pipelines. 
Particularly, recent work on in-processing fairness for DL models mainly falls under two types of approaches: constraint-based and adversarial learning methods \cite{tian2024multifair}. 
Constraint-based methods usually incorporate fairness metrics directly into the model optimisation objectives as constraints or regularisation terms. 
For instance, Zemel \etal{} \cite{zemel2013learning}, the pioneer in this direction, put demographic parity constraints on model predictions. 
Subsequent work also includes using approximations \cite{xu21to} or modified training schemes \cite{padala2021fnnc} to improve scalability. 
Adversarial methods intend to learn representations as fairly as possible by removing sensitive attribute information. 
In such procedures, additional prediction heads may be introduced for attribute subgroup predictions, and the information concerning sensitive attributes would be removed through inverse gradient updating \cite{wang2020mitigating,karkkainen2021fairface} or disentangling features \cite{jung2021fair,locatello2019fairness,sarhan2020fairness,guo2022learning}. 
Other fairness enhancing techniques include data augmentations \cite{mo2021object}, sampling \cite{roh2021fairbatch,khalili2021fair}, data noising \cite{zhang2021balancing}, dataset balancing with generative methods \cite{hwang2020unsupervised,joo2020gender,ramaswamy2021fair}, and reweighting mechanisms \cite{zhao2020maintaining,gong2021mitigating}. 
Recently, mixup operations \cite{verma2019manifold,zhang2018mixup,tian2024multifair} are adopted to enhance fairness by blending inputs across subgroups \cite{chuang2021fair,du2021fairness}. 
Yet most of these studies focus on protecting one sensitive attribute and are hardly able to handle several sensitive attributes all at once. 
And multi-attribute fairness protection remains relatively rarely explored.

\subsection{Existing fairness metrics and multi-attribute fairness protection}

The well-known fairness metrics are generally divided into group fairness---such as demographic parity (DP), equality of opportunity (EO), and predictive quality parity (PQP)---and individual fairness \cite{dwork2012fairness,berk2021fairness,vzliobaite2017measuring,joseph2016fairness,pleiss2017fairness}. 
The former mainly focuses on statistical/demographic equality among groups defined by sensitive attributes, while the latter cares more about the principle that `similar individuals should be evaluated or treated similarly.' 
However, satisfying fairness metrics all at once is hard to achieve because they are usually not compatible with each other \cite{barocas2023fairness}. 
In practice, it may need to deliberate on the choice of the specified distance in individual fairness \cite{joseph2016fairness,dwork2012fairness}. 
Moreover, the three commonly used group fairness measures (that is, DP, EO, and PQP) can only deal with one single sensitive attribute with binary values. 
Although extending them to scenarios of one sensitive attribute with multiple values is possible like statistical parity \cite{agarwal2019fair,jiang2020wasserstein}, they are still limited when facing several sensitive attributes at the same time. 
Recent work includes a newly proposed fairness measure named discriminative risk (DR) \cite{bian2023increasing_re} that is capable of capturing bias from both individual and group fairness aspects and two fairness frameworks (that is, InfoFair \cite{kang2022infofair} and MultiFair \cite{tian2024multifair}) to deliver fair predictions in face of multiple sensitive attributes. 
Yet these two fairness frameworks are not quantitative measures that directly evaluate the discrimination level of ML models.

\section{Methodology}
In this section, we formally study the measurement of fairness from a manifold perspective. Some standard notations that we use in this paper is listed in Table~\ref{tab:notation}. 
\begin{table}[h]
\centering
\caption{Mathematical notation summary}
\label{tab:notation}
\vspace{-1.2em}
\begin{tabular}{cl}
\hline
\textbf{Notation} & \textbf{Meaning} \\
\hline
$x$ & Scalar (italic lowercase letter) \\
$\bm{x}$ & Vector (bold lowercase letter) \\
$X$ & Matrix or set (italic uppercase letter) \\
$\mathsf{X}$ & Random variable (serif uppercase letter) \\
$\mathbb{R}$ & Real numbers \\
$\mathbb{Z}$, $\mathbb{Z}^+$ & Integers, and positive integers \\
$\mathbb{P}(\cdot)$ & Probability measure \\
$\mathbb{E}(\cdot)$ & Expectation \\
$\mathbb{V}(\cdot)$ & Variance of a random variable \\
$\mathbb{I}(\cdot)$ & Indicator function \\
$\mathcal{F}$ & Hypothesis space \\
$f(\cdot)$ & Models in one hypothesis space \\
$[n]$ & The set $\{1,2,...,n\}$ for brevity\\
\hline
\end{tabular}
\end{table}

%
%
%

Furthermore, we define:
\begin{itemize}
    \item $S = \{(\bm{x}_i,y_i)\}_{i=1}^n$ denotes a dataset with i.i.d. instances drawn from a feature-label space $\mathcal{X} \times \mathcal{Y}$ based on an unknown distribution.
    
    \item The feature (or input) space $\mathcal{X}$ is arbitrary, while the label (or output) space $\mathcal{Y} \!=\! \{1,2,\ldots,n_c\}$ ($n_c \geqslant 2$) is finite, supporting both binary and multi-class classification.
    
    \item For dataset $S$ containing sensitive attributes (SAs), each instance 
    is represented as $\bm{x}\! \defineq \! (\xneg,\xpos)$, where:
    \begin{itemize}
        \item $\xneg = [x_1, x_2, \ldots, x_{n_x}]^{\mathsf{T}}$ represents non-sensitive features, and $n_x$ is the number of non-sensitive features;
        \item $\xpos = [a_1, a_2, \ldots, a_{n_a}]^{\mathsf{T}}$ represents SAs, and the number of SAs $n_a\geqslant 1$, allowing multiple attributes; and
        \item $a_i \in \mathbb{Z}_+$ for the $i$-th SA ($1 \leqslant i \leqslant n_a$), allowing both binary and multiple values.
    \end{itemize}
    
    \item A hypothesis function $f \in \mathcal{F}: \mathcal{X} \mapsto \mathcal{Y}$ maps from the feature space to the label space, where:
    \begin{itemize}
        \item $\mathcal{F}$ is the hypothesis space, and
        \item $f(\bm{x})$ or $\hat{y}$ denotes the prediction for instance $\bm{x}$.
    \end{itemize}
\end{itemize}

\subsection{Model fairness assessment from a manifold perspective}
\label{sec:method,dist}

Given the dataset $S\!=\!\{ (\xneg_i,\xpos_i,y_i) \mid i\!\in\![n]\}$ composed of instances including SAs, here we denote one instance by $\insx\!=\!(\xneg,\xpos)\!=\! [x_1,...,x_{n_x}, a_1,...,a_{n_a}]^\mathsf{T}$ for clarity, where $n_a$ is the number of sensitive/protected attributes and $n_x$ is that of non-sensitive/unprotected attributes in $\insx$. 
In this paper, we introduce new fairness measures in scenarios for several SAs with multiple possible values. 
Note that the proposed fairness measure here is extended from our previous work---a fairness measure in scenarios for SAs with binary values~\cite{bian2024does}.

\subsubsection{Distance between sets for one bi-valued SA, from our previous work \cite{bian2024does}}

Inspired by the principle of individual fairness---similar treatment for similar individuals, \emph{if we view the instances (with the same SAs) as data points on certain manifolds, the manifold representing members from the marginalised/unprivileged group(s) is supposed to be as close as possible to that representing members from the privileged group}. 
To measure the fairness with respect to the SA, we have proposed a fairness measure that is inspired by 
`the distance of sets' introduced in mathematics.\footnote{
The distance between sets used here is known as the Hausdorff distance (HD), denoted as $\newdist_H$. It measures 
the distance between two non-empty subsets (namely $X$ and $Y$) in a metric space 
$(\mathcal{M},\newdist)$, it is defined as $$
\topequation\footnotesize
\newdist_H \defineq \max\big\{
     \sup_{x\in X} \inf_{y\in Y} \newdist(x,y), \\
     \sup_{y\in Y} \inf_{x\in X} \newdist(x,y)
    \big\} \,.$$ When dealing with discrete sets, we may replace supremum and infimum with maximum and minimum respectively.

We chose it over any other distance for two reasons. One is that if a whole data space is divided by one bi-valued SA, the obtained subspaces are assumed to be disjoint, and the whole data space is viewed as a Hausdorff space, a premise when using the HD. The other is that the HD is the greatest of all the distances from one point in one set to the closest point in the other set, which helps us not be misled by underestimated or overly optimistic results, considering that we were trying to capture the discrimination within. 
} 
For a certain bi-valued SA $a_i\!\in\!\mathcal{A}_i \!=\!\{0,1\}$, $S$ can be divided into two subsets $S_1\!=\! \{(\insx,y)\!\in S\mid a_i\!=\!1\}$ and $\bar{S}_1\!=\! S\setminus S_1 \!=\! \{(\insx,y)\!\in S\mid a_i\!\neq\! 1\}$, where $a_i\!=\!1$ means the corresponding instance is a member of the privileged group. 
%
Then given a specific distance metric $\newdist(\cdot,\cdot)$\footnote{%
Here we use the standard Euclidean metric. In fact, any two metrics $\newdist_1,\newdist_2$ derived from norms on the Euclidean space $\mathbb{R}^d$ are equivalent in the sense that there are positive constants $c_1,c_2$ such that $c_1\newdist_1(x,y)\leqslant \newdist_2(x,y)\leqslant c_2\newdist_1(x,y)$ for all $x,y\in \mathbb{R}^d$.
} on the feature space, our previous distance between these two subsets (that is, $S_1$ and $\bar{S}_1$) is defined by
\begin{equation}
\topequation\small
\begin{split}
\newDist(S_1,\bar{S}_1) \defineq
\max\big\{
& \max_{(\insx,y)\in S_1} \min_{(\insx',y')\in \bar{S}_1} \newdist\big( (\xneg,y),({\xneg}',y') \big) ,\\
& \max_{(\insx',y')\in \bar{S}_1} \min_{(\insx,y)\in S_1} \newdist\big( (\xneg,y),({\xneg}',y') \big) 
\big\} \,,\label{eq:1}
\end{split}
\end{equation}%
%
and it is viewed as the distance between the manifolds of marginalised group(s) and that of the privileged group. 
Notice that this distance satisfies three basic properties: identity, symmetry, and triangle inequality.\footnote{%
Notice that the distance defined in Eq.~\eqref{eq:1} satisfies the following basic properties: 
1) For any two data sets $S_1$ and $S_2\in \mathcal{X\times Y}$, $\newDist(S_1,S_2)=0$ if and only if $S_1$ equals $S_2$, also known as identity; 
2) For any two sets $S_1$ and $S_2$, $\newDist(S_1,S_2)= \newDist(S_2,S_1)$, also known as symmetry; and 
3) For any sets $S_1$, $S_2$, and $S_3$, we have the triangle inequality $\newDist(S_1,S_3)\leqslant \newDist(S_1,S_2)+ \newDist(S_2,S_3)$. 
} Analogously, for a trained classifier $f(\cdot)$, we can calculate
\begin{equation}
\topequation
\small\begin{split}
    \newFist(S_1,\bar{S}_1)= \max\big\{
    & \max_{(\insx,y)\in S_1} \min_{(\insx',y')\in \bar{S}_1} \newdist\big( (\xneg,\hat{y}),({\xneg}',\hat{y}') \big) ,\\
    & \max_{(\insx',y')\in \bar{S}_1} \min_{(\insx,y)\in S_1} \newdist\big( (\xneg,\hat{y}),({\xneg}',\hat{y}') \big)
    \big\} \,.
\end{split}\label{eq:2}
\end{equation}%
%
For simplification, we could rewrite Equations~\eqref{eq:1} and \eqref{eq:2} as
\begin{equation}
\topequation
\small\begin{split}
\extBist \defineq \max\big\{
    & \max_{(\insx,y)\in S_1} \min_{(\insx',y')\in \bar{S}_1} 
    \extprev{\newy},\\
    & \max_{(\insx',y')\in \bar{S}_1} \min_{(\insx,y)\in S_1} 
    \extprev{\newy}
\big\} \,,
\end{split}\label{eq:3}
\end{equation}%
by recording the true label $y$ and the prediction $\hat{y}$ as one denotation (say $\newy$). 
%
We will continue using the above notations in the subsequent context.

\subsubsection{Distance between sets for one multi-valued SA}
As for the scenarios where only one SA exists, 
let $\xpos=[a_i]^\mathsf{T}$ be a single SA, in other words, $n_a\!=\!1$, $a_i\!\in\!\mathcal{A}_i \!=\! \{1,2,...,n_{a_i}\}$, 
$n_{a_i}\!\geqslant 3$, and $n_{a_i}\!\in\! \mathbb{Z}_+$. 
Then the original dataset $S$ can be divided into a few disjoint sets according to the value of this attribute $a_i$, that is, $S_j=\{(\insx,y)\in S\mid a_i=j\}, \forall j\in\mathcal{A}_i$. 
%
We can now extend Eq.~\eqref{eq:3} and introduce the following distance measures: 
(i) \emph{maximal distance measure for one SA,}
\begin{equation}
\topequation\small 
\extDist{} \!\defineq\!
\max_{1\leqslant j\leqslant n_{a_i}} 
\!\! \big\{
    %
    \max_{(\insx,y)\in S_j} \!
    \min_{(\insx',y')\in \bar{S}_j}
    \! \extprev{\newy}
\big\}
\,,\label{eq:4a}
\end{equation}
and (ii) \emph{average distance measure for one SA,}
\begin{equation}
\topequation\small
\textstyle
\extDist{avg} \!\defineq\!
\frac{1}{n}\sum_{j=1}^{n_{a_i}} 
%
\sum_{(\insx,y)\in S_j} 
\min_{(\insx',y')\in \bar{S}_j}
\extprev{\newy} 
\,,\label{eq:4b}
\end{equation} 
where $\bar{S}_j\!= S\!\setminus\! S_j$. 
Notice a special case that $\extDist{} \!=\! \extBist$ when $\mathcal{A}_i \!=\! \{0,1\}$. 

\subsubsection{Distance between sets for several multi-valued SAs}
Now we discuss the general case, where we have several SAs $\xpos\!=\![a_1,a_2,...,a_{n_a}]^\mathsf{T}$ and each $a_i\!\in\!\mathcal{A}_i\!=\!\{1,2,..,n_{a_i}\}$, where $n_{a_i}$ is the number of values for this SA $a_i\, (1\!\leqslant\! i\!\leqslant\! n_a)$. 
We can now introduce the following generalised distance measures: 
(i) \emph{maximal distance measure for multiple SAs,}
\begin{equation}
\topequation\textstyle
    \small \extEist{}\defineq 
    \max_{1\leqslant i\leqslant n_a}
    \extDist{} \,,\label{eq:6a}
\end{equation}
and (ii) \emph{average distance measure for multiple SAs,}
\begin{equation}
\topequation\textstyle
    \small \extEist{avg} \defineq
    \frac{1}{n_a} \sum_{i=1}^{n_a}
    \extDist{avg} \,.\label{eq:6b}
\end{equation}

\textbf{Remarks.}
\begin{enumerate}[label=(\arabic*)]
\item It is easy to see that $\extEist{} \geqslant \extEist{avg}$.
\item Both $\extDist{}$ and $\extDist{avg}$ measure the fairness regarding the SA $a_i$. 
\item As their names suggest, the maximal distance represents the largest possible disparity between instances with different SAs, while the average distance reflects the average disparity between instances with different SAs. 
The formal distance measures are more stringent, they are susceptible to data noise. 
In contrast, the latter type of distance measures are more resilient against the influence of data noise.\looseness=-1 
\end{enumerate}

We remark that $\newDist_{\xpos}(S)$, $\newDist_{\xpos}^\text{avg}(S)$ reflect the biases from the data and that $\newDist_{f,\xpos}(S)$, $\newDist_{f,\xpos}^\text{avg}(S)$ reflect the biases from the learning algorithm. 
Then, the following values could be used to capture the intensity of variation between them, \ie{}
\begin{subequations}
\small
\begin{align}
\newfist(f)&= \log\left(
    \frac{\newDist_{f,\xpos}(S)}{\newDist_{\xpos}(S)}
\right) \,,\label{eq:8a}\\
\newfist^\text{avg}(f) &=
\log\left( \frac{ 
    \newDist_{f,\xpos}^\text{avg}(S) }{ 
    \newDist_{\xpos}^\text{avg}(S) 
}\right) \,.\label{eq:8b}
\end{align}%
\label{eq:8}%
\end{subequations}%
That is to say, we use them to capture the extra discrimination introduced in the learning procedure, reflecting the fairness degree of this classifier. 
We name the fairness degrees defined as above of one classifier by Eq.~\eqref{eq:8} as `\emph{maximum \ppsdisfull{} (}\ppsdisabbr\emph{)}' and `\emph{average} \ppsdisabbr{}', respectively.

\subsection{An efficient approximation of distances between sets for Euclidean spaces}
To reduce the high computational complexity ($\mathcal{O}(n^2)$) of directly calculating Equations~\eqref{eq:4a} and \eqref{eq:4b},\footnote{
The original time complexity of the direct computation in Eq.~\eqref{eq:3} given a dataset of size $n$ is $\mathcal{O}(2n_1(n-n_1))$ where $n_1$ is the size of $S_1$, that is, $\mathcal{O}(n^2)$. 
By extension, the time complexity of the direct computation in Eq.~(4) and (5) is $\mathcal{O}(\sum_{j=1}^{n_{a_i}} n_j(n-n_j))$ where $n_j$ is the size of $S_j$ for all $j\in\{1,2,...,n_{a_i}\}$, that is, $\mathcal{O}(n^2)$ as well. 
} we propose an efficient algorithm that can be viewed as a modification of an $\mathcal{O}(n\log n)$-algorithm in our previous work \cite{bian2024does}. 

We start by recalling the algorithm introduced in \cite{bian2024does}. 
Since the core operation in Equations~\eqref{eq:4a} and \eqref{eq:4b} is to evaluate the distance between data points inside $\mathcal{X\times Y}$, to reduce the number of distance evaluation operations involved in Eq.~\eqref{eq:4a} and \eqref{eq:4b}, we observe that the distance between similar data points tends to be closer than others after projecting them onto a general one-dimensional linear subspace (refer \cite[Lemma 1]{bian2024does}). 
To be concrete, let $g:\mathcal{X\times Y}\mapsto\mathbb{R}$ be a random projection, then we could write $g$ as
\begin{equation}
\topequation
    \small g(\insx,\newy;\vecw)=
    g(\xneg,\xpos,\newy;\vecw)=
    [\newy,x_1,...,x_{n_x}]^\mathsf{T}
    \vecw \,,
    \label{eq: uniformized projection}
\end{equation}%
where $\vecw\!=\! [w_0,w_1,...,w_{n_x}]^\mathsf{T}$ is a non-zero random vector. 
Now, we choose a random projection $g\!:\!\mathcal{X\!\times\! Y}\!\mapsto\!\mathbb{R}$, then we sort all the projected data points on $\mathbb{R}$. 
According to \cite[Lemma 1]{bian2024does}, it is likely that for the instance $(\insx,y)$ in $S_j$, the desired instance 
$\argmin_{(\insx',y')\in \bar{S}_j} \extprev{y}$ would be somewhere near it after the projection, and vice versa. 
Thus, by using the projections in Eq.~\eqref{eq: uniformized projection}, 
we could accelerate the process in Eq.~\eqref{eq:4a} and \eqref{eq:4b} by checking several adjacent instances rather than traversing the whole dataset.

In this paper, instead of taking one random vector each time, we now take a few orthogonal random vectors each time and do the above process for all these orthogonal vectors. 
The number of these orthogonal vectors could be $n_x+1$, or smaller (such as two or three), if the practitioners would like to save more time in practice. 
For instance, we set two orthogonal random vectors\footnote{%
In practice, we first generate a ($1+n_x$)-dimensional matrix randomly (for example, from the standard normal distribution), and then use the QR decomposition to get $Q$---an orthogonal matrix. The columns of $Q$ are orthogonal unit vectors, and the first two (or three, depending on the number of orthogonal vectors that the practitioner likes) columns will be used as the orthogonal vectors in line~\ref{line:app,2} of Algorithm~\ref{alg:approx}.\looseness=-1 
} in Algorithm~\ref{alg:approx} at present. 
Then we take the minimum among all estimated distances. 
This modification may slightly increase the time cost of approximation a bit compared with our previous work \cite[Algorithm 1]{bian2024does}, yet will still significantly accelerate the execution speed and the effectiveness of the projection algorithm, compared with the direct calculation of distances.

\begin{algorithm}[t]\small
\caption{\small
\ppsuperfull, \aka{} \texttt{\ppsuperabbr}$(\{(\xneg_i,\xpos_i)\}_{i=1}^n, \{\newy_i\}_{i=1}^n; m_1,m_2)$, 
}\label{alg:extend}
\begin{algorithmic}[1]
\REQUIRE Dataset $S=\{(\insx_i,y_i)\}_{i=1}^n =\{(\xneg_i,\xpos_i,y_i)\}_{i=1}^n$ where $\xpos_i=[a_{i,1},a_{i,2},...,a_{i,n_a}]^\mathsf{T}$, 
prediction of $S$ by the classifier $f(\cdot)$ that has been trained, that is, $\{\hat{y}_i\}_{i=1}^n$, and two hyperparameters $m_1$ and $m_2$ as the designated numbers for repetition and comparison respectively
\ENSURE Approximation of $\extEist{}$ and $\extEist{avg}$
\FOR{$j$ from $1$ to $n_a$}
\label{line:sup,1}
    \STATE $d_\text{max}^{(j)}, d_\text{avg}^{(j)}=$ \texttt{\ppsalgabbr}$(\{(\xneg_i,a_{i,j})\}_{i=1}^n, \{\newy_i\}_{i=1}^n; m_1,m_2)$
\ENDFOR
\label{line:sup,3}
\RETURN $\max_{1\leqslant j\leqslant n_a} \{d_\text{max}^{(j)} \mid j\in[n_a]\}$ and $\frac{1}{n_a} \sum_{j=1}^{n_a} d_\text{avg}^{(j)}$
\end{algorithmic}
\end{algorithm}

\begin{algorithm}[t]\small
\caption{\small
\ppsalgfull{}, \aka{} 
\texttt{\ppsalgabbr}$(\{(\xneg_i,a_i)\}_{i=1}^n$, $\{\newy_i\}_{i=1}^n; m_1,m_2)$
}\label{alg:approx}
\begin{algorithmic}[1]
\REQUIRE Dataset $S=\{(\insx_i,y_i)\}_{i=1}^n =\{(\xneg_i,\xpos_i,y_i)\}_{i=1}^n$, prediction of $S$ by the classifier $f(\cdot)$ that has been trained, that is, $\{\hat{y}_i\}_{i=1}^n$, and hyperparameters $m_1$, $m_2$ as the designated number for repetition and comparision
\ENSURE Approximation of $\extDist{}$ and $\extDist{avg}$
\FOR{$j$ from $1$ to $m_1$}
\label{line:app,1}
    \STATE Take two orthogonal vectors $\vecw_0$ and $\vecw_1$ where each $\vecw_k\in[-1,+1]^{1+n_x} \,(k=\{0,1\})$ 
    \label{line:app,2}
    \FOR{$k$ from $0$ to $1$}
    \STATE $t_\text{max}^k, t_\text{avg}^k=\!$\,\texttt{\ppssubabbr}$(\{(\xneg_i,a_i)\}_{i=1}^n, \{\newy_i\}_{i=1}^n, \vecw_k;m_2)$
    \ENDFOR
    \label{line:app,5}
    \STATE $d_\text{max}^j= \min\{ t_\text{max}^k \mid k\in\{0,1\} \} =\min\{t_\text{max}^0, t_\text{max}^1\}$
    \STATE $d_\text{avg}^j\,= \min\{ t_\text{avg}^k \mid k\in\{0,1\} \} \,=\min\{t_\text{avg}^0, t_\text{avg}^1\}$
\ENDFOR
\label{line:app,8}
\RETURN $\min\{d_\text{max}^j \mid j\in[m_1]\}$ and $\frac{1}{n} \min\{d_\text{avg}^j \mid j\in[m_1]\}$ \label{line:alg2,l9}
\end{algorithmic}
\end{algorithm}

\begin{algorithm}[t]\small
\caption{\small
\ppssubfull{} in approximation, \aka{} 
\texttt{\ppssubabbr}$(\{(\xneg_i,a_i)\}_{i=1}^n, \{\newy_i\}_{i=1}^n, \vecw;m_2)$
}\label{alg:accele}
\begin{algorithmic}[1]
\REQUIRE Data points $\{(\xneg_i,a_i)\}_{i=1}^n$, its corresponding value $\{\newy_i\}_{i=1}^n$, where $\newy_i$ could be its true label $y_i$ or prediction $\hat{y}_i$ by the classifier $f(\cdot)$, a random vector $\vecw$ for projection, and a hyperparameter $m_2$ as the designated number for comparison
\ENSURE Approximation of $\extDist{}$ and $n\extDist{avg}$
\STATE Project data points onto a one-dimensional space based on Eq.~\eqref{eq: uniformized projection}, in order to obtain $\{g(\insx_i,\newy_i;\vecw)\}_{i=1}^n$
\label{line:acc,1}
\STATE Sort original data points based on $\{g(\insx_i,\newy_i;\vecw)\}_{i=1}^n$ as their corresponding values, in ascending order
\label{alg:acc,l2}
\FOR{$i$ from $1$ to $n$}
\label{line:acc,3}
    \STATE Set the anchor data point $(\insx_i,\newy_i)$ in this round
    \label{line:acc,4}
\LineComment{If $a_i=j$ (marked for clarity), in order to approximate $\min_{(\insx',y')\in \bar{S}_j} \newdist\big( (\xneg_i,\newy_i),({\xneg}',\newy') \big)$ }
    \STATE Compute the distances $\newdist((\xneg_i,\newy_i),\cdot)$ for at most $m_2$ nearby data points that meets $a\neq a_i$ and $g(\xneg,\newy;\vecw) \leqslant g(\xneg_i,\newy_i;\vecw)$
    \STATE Find the minimum among them, recorded as $d_\text{min}^s$
    \STATE Compute the distances $\newdist((\xneg_i,\newy_i),\cdot)$ for at most $m_2$ nearby data points that meets $a\neq a_i$ and $g(\insx,\newy;\vecw) \geqslant g(\insx_i,\newy_i;\vecw)$
    \STATE Find the minimum among them, recorded as $d_\text{min}^r$
    \STATE $d_\text{min}^{(i)}= \min\{d_\text{min}^s, d_\text{min}^r\}$
    \label{line:acc,9}
\ENDFOR
\label{line:acc,11}
\RETURN $\max \{d_\text{min}^{(i)} \mid i\in[n] \}$ and $\sum_{i=1}^n d_\text{min}^{(i)}$
\end{algorithmic}
\end{algorithm}

Then we could propose an approximation algorithm to estimate the distance between sets in Equations~\eqref{eq:4a} and \eqref{eq:4b}, named as `\emph{\ppsalgfull{} (}\ppsalgabbr\emph{)}', shown in Algorithm~\ref{alg:approx}. 
As for the distance in Equations~\eqref{eq:6a} and \eqref{eq:6b}, we propose `\emph{\ppsuperfull{} (}\ppsuperabbr\emph{)}', shown in Algorithm~\ref{alg:extend}. 
Note that there exists a sub-route within \ppsalgabbr{} to obtain an approximated distance between sets, which is named as `\emph{\ppssubfull{} (}\ppssubabbr\emph{)}' and shown in Algorithm~\ref{alg:accele}. 
As the time complexity of sorting in line~\ref{alg:acc,l2} of Algorithm~\ref{alg:accele} could reach $\mathcal{O}(n\log n)$, we could get the computational complexity of Algorithm~\ref{alg:accele} as follows: 
i) The complexity of line~\ref{line:acc,1} is $\comp(n)$; and 
ii) The complexity from line~\ref{line:acc,4} to line~\ref{line:acc,9} is $\comp(2m_2+1)$. 
Thus the overall time complexity of Algorithm~\ref{alg:accele} would be $\comp(n(\log n+m_2+1))$, 
and that of Algorithm~\ref{alg:approx} be $\comp(m_1n(\log n+m_2))$, and that of Algorithm~\ref{alg:extend} be $\comp(n_am_1n(\log n+m_2))$. 
As both $m_1$ and $m_2$ are the designated constants, 
and $n_a$ is also a fixed constant for one specific dataset, 
the time complexity of computing the distance is then down to $\comp(n\log n)$, which is more welcome than $\comp(n^2)$ for the direct computation in Section~\ref{sec:method,dist}.

It is worth noting that in line \ref{line:alg2,l9} of Algorithm~\ref{alg:approx}, we use the minimal instead of their average value. 
The reason is that in each projection, the exact distance for one instance would not be larger than the calculated distance for it via \ppssubabbr{}; and the same observation holds for all of the projections in \ppsalgabbr{}. 
Thus, the calculated distance via \ppsalgabbr{} is always no less than the exact distance, and the minimal operator should be taken finally after multiple projections.

\subsection{Algorithmic effectiveness analysis of \ppsalgabbr}
\label{subsec:method,analysis}
As \ppsalgabbr{} in Algorithm~\ref{alg:approx} is the core component devised to facilitate the approximation of direct calculation of the distance between sets, in this subsection, we detail more about its algorithmic effectiveness under some conditions.

In higher dimensional statistics, it is known that if we take a random projection $A$ from a n-dimensional space to a $k$-dimensional subspace, then for every vector $\mathbf{x}$, we have the following probability estimation:
\begin{equation}
\topequation\small 
\mathbb{P}\left(\left|\frac{|A\mathbf{x}|}{|\mathbf{x}|} - \sqrt{\frac{k}{n}}\right| \leq \varepsilon\sqrt{\frac{k}{n}}\right) \geq 1-C\exp\{-C k\varepsilon^2\} \,,
    \label{eq: higher dim statistics}
\end{equation}%
where $C$ is a universal constant independent of the dimensional $n,k$ and the projection $A$. In particular, when $k=1$, we have proved a simplified version of \eqref{eq: higher dim statistics} (refer to \cite[Lemma 1]{bian2024does}), restated in Lemma~\ref{projection 1}.  These results support the following observation: \emph{`the distance between similar data points tends to be closer than others after projecting them onto a general one-dimensional linear subspace.'}
It demonstrates by Eq.~\eqref{eq: ineq conclusion} that the probability $\mathbb{P}(\bm{v}_1,\bm{v}_2)$ also goes to zero when the ratio $\sfrac{r_1}{r_2}$ goes to zero. 
Additionally, it is easy to observe that $\mathbb{P}(\bm{v}_1,\bm{v}_2)$ reaches the same order of magnitude as $\sfrac{r_1}{r_2}$, and especially, when $r_1$ equals $r_2$, $\mathbb{P}(\bm{v}_1,\bm{v}_2)$ could be roughly viewed as $\sfrac{1}{2}$ for coarse approximation. 
It means that the breaking probability of the aforementioned statement---similar data points leading to closer distances---tends to increase as $r_1$ gradually gets closer to $r_2$. 
And the profound meaning behind Lemma~\ref{projection 1} is that the bigger the gap of lengths between $\bm{v}_1$ and $\bm{v}_2$ is, the more effective and efficient our proposed approximation algorithms would be.

\begin{lemma}[Lemma 1 \cite{bian2024does}]
\label{projection 1}
Let $\bm{v}_1$ (resp. $\bm{v}_2$) be a vector in the $n$-dimensional Euclidean space $\mathbb{R}^n$ with length $r_1$ (resp. $r_2$) such that $r_1\leqslant r_2$. Let $\vecw\subset \mathbb{R}^n$ be a unit vector. We define $\mathbb{P}(\bm{v}_1,\bm{v}_2)$ as the probability that $|\langle \vecw,\bm{v}_1\rangle|\geqslant |\langle \vecw,\bm{v}_2\rangle|$. Then,
\begin{equation}\label{eq: ineq conclusion}
\topequation\small 
\frac{\sin\phi}{\pi} \cdot \frac{r_1}{r_2} 
\leqslant \mathbb{P}(\bm{v}_1,\bm{v}_2) \leqslant
\bigg( 1+\frac{r_1^2}{r_2^2}
\bigg)^{-\sfrac{1}{2}} \cdot\frac{r_1}{r_2} \,,
\end{equation} where $\phi$ represents the angle between $\bm{v}_1$ and $\bm{v}_2$\,.
\end{lemma}

Our main result in this subsection is Proposition~\ref{density,alt} (a modified version of \cite[Proposition 2]{bian2024does}), whereby Eq.~\eqref{eq: probability}, the efficiency of \ppsalgabbr{} decreases as the scaled density $\mu$ of the original dataset increases. 
Meanwhile, when dealing with large-scale datasets, the more insensitive attributes we have, the more efficient \ppsalgabbr{} is. 
In general, the efficiency of \ppsalgabbr{} depends on the shape of these two subsets of $S$. 
Roughly speaking, the more separated these two sets are from each other, the more efficient \ppsalgabbr{} is.\looseness=-1

\begin{proposition}
\label{density,alt}
Let $S\!\!=\! \{(\bm{x}_i,y_i)\}_{i=1}^n \!\subset\! \mathcal{X}\!\times \!\mathcal{Y}$ be a $(k\!+\!1)$-dimensional dataset where instances have $(k\!+\!1)$ features, 
an evenly distributed dataset with a size of $n$ that is a random draw of the feature-label space $\mathcal{X\!\times\! Y}$. 
For any two subsets of $S$ with distance $d$ (ref. Eq.~\eqref{eq:3} and \eqref{eq:4a}), suppose further that the scaled density    
\begin{equation}\label{eq: scaled density}
    \topequation\small
    \limsup_{\mathbf{B}\subset \mathbb{R}^{k+1} \text{ an Euclidean ball}}\frac{1}{\mathrm{Vol}(\mathbf{B})}\#(\mathbf{B}\cap S)= \frac{\mu}{\mathrm{Vol}(\mathbf{B}(d))}\,,
\end{equation}
for some positive real number $\mu$ (here $\#$ denotes the number of points of a finite set and $\mathbf{B}(d)$ denotes a ball of radius $d$). Then, with probability at least 
\begin{equation}\label{eq: probability}
\topequation\small
     1-\bigg(\frac{\pi\mu}{m_2\mathrm{Vol}(\mathbf{B}(1))} 
     \Big(\Big(1+\frac{n}{\mu}\Big)^{\frac{1}{k+1}}-\alpha\Big)\bigg)^{m_1}\,,
\end{equation}
\ppsalgabbr{} could reach an approximate solution that is at most $\alpha$ times of the distance between these two subsets.
\end{proposition}
\begin{proof}%
Let $S_1,S_2,...,S_{n_{a_i}}$ be $n_{a_i}$ sub-datasets of $S$. For each $j\in\{1,2,...,n_{a_i}\}$, then $S_j$ and $\bar{S}_j$ be two sub-datasets of $S$. 
We fix the instance $\bm{v}_j\in S_j$ such that $d\defineq \mathbf{D}_{\xpos}(S,a_i) =\mathbf{d}(\bm{v}_j,\bm{v}_k)$ for some $\bm{v}_k\in \bar{S}_j$. 
For simplicity, we may set $\bm{v}_j$ as the origin. 
The probability that an instance $\bm{v}\in \bar{S}_j$ has a shorter length than $\bm{v}_k$ after projection to a line (see Eq.~\eqref{eq: uniformized projection}) is denoted as $\mathbb{P}(\bm{v}_k,\bm{v})$. 
By assumption, we only need to consider those instances whose length is greater than $\alpha d$ (outside the ball $\mathbf{B}(\alpha d)$ centred at the origin). Hence, the desired probability is bounded from below by
\begin{equation}
\label{eq: probability 1}
\topequation\small
1-\bigg( \frac{1}{m_2}
\sum_{\bm{v}\notin \mathbf{B}(\alpha d)}
\mathbb{P}(\bm{v}_j,\bm{v})
\bigg)^{m_1} \,.
\end{equation}
However, Eq.~\eqref{eq: probability 1} is based on the extreme assumption that all instances lie on the same two-dimensional plane. In our case, the instances are evenly distributed. Hence, we may adjust the probability by multiplying 
\begin{equation}
\topequation\small 
\frac{\mathrm{Vol}(S^1(\frac{\lVert \bm{v}\rVert}{d}))}{\mathrm{Vol}(S^k(\frac{\lVert \bm{v}\rVert}{d}))}
=\frac{\Gamma(\frac{k+1}{2})}{\pi^{\frac{k-1}{2}}} \cdot 
\bigg( \frac{d}{\lVert \bm{v}\rVert} \bigg)^{k-1} \,,\nonumber
\end{equation}
where $\Gamma(\cdot)$ denotes the Gamma function and $\mathrm{Vol}(S^i(r))$ denotes the area of the $i$-th dimensional sphere of radius $r$. 
Hence, by Lemma~\ref{projection 1}, 
the desired probability is lower bounded by 
\begin{equation}
\label{eq: probability 11}
\topequation\small
    1-\bigg(\frac{1}{m_2} \sum_{\bm{v}\notin \mathbf{B}(\alpha d)} \bigg(1+\frac{d^2}{\lVert \bm{v}\rVert^2}\bigg)^{-\frac{1}{2}} \cdot
    \frac{\Gamma(\frac{k+1}{2})}{\pi^{\frac{k-1}{2}}}
    \cdot\bigg(\frac{d}{\lVert \bm{v}\rVert}\bigg)^k\bigg)^{m_1}\,.
\end{equation}
Under our assumption, Eq.~\eqref{eq: probability 11} attains the lowest value when the data are evenly distributed inside a hollow ball $\mathbf{B}_j\setminus \mathbf{B}(d)$ centred at $\bm{v}_j$. The radius of $\mathbf{B}_j$, denoted as $r_j$, satisfies 
\begin{equation}\label{eq: radius}
\topequation\small
    n-1=\mu\frac{\mathrm{Vol}(\mathbf{B}_j\setminus \mathbf{B}(d))}{\mathrm{Vol}(\mathbf{B}(d))}=\mu\Big(\Big(\frac{r_j}{d}\Big)^{k+1}-1\Big)\,.
\end{equation}
In this situation, we may write the summation part of Eq.~\eqref{eq: probability 11} as an integration. To be more specific, Eq.~\eqref{eq: probability 11} is lower bounded by
\begin{equation}\label{eq: probability 22}
\topequation\small
    1-\bigg(\frac{1}{m_2} \int_{\alpha d}^{r_j} A(x)\mu \mathrm{Vol}(S^{k}(x))dx \bigg)^{m_1}\,,
\end{equation}
where $A(x)=(1+\frac{d^2}{x^2})^{-\frac{1}{2}} \frac{\Gamma(\frac{k+1}{2})}{\pi^{(k-1)/2}} \cdot(\frac{d}{x})^k$. 
Moreover, Eq.~\eqref{eq: probability 22} can be simplified as 
\begin{equation}\label{eq: probability 33}
\topequation\small
    1-\bigg(\frac{1}{m_2\mathrm{Vol}(\mathbf{B}(1))} \int_{\alpha d}^{r_j} \frac{\pi\mu}{d}\cdot \frac{x}{\sqrt{x^2+d^2}} dx \bigg)^{m_1}\,.
\end{equation}
Combining Eq.~\eqref{eq: radius} and \eqref{eq: probability 33}, we conclude that the desired probability is lower bounded by 
\begin{equation}\label{eq: probability 44}
\topequation\small
1-\bigg(
\frac{\pi\mu}{m_2\mathrm{Vol}(\mathbf{B}(1))} 
\bigg( 
    \Big(\Big(1+\frac{n}{\mu}\Big)^{\frac{2}{k+1}} 
    +1\Big)^{\frac{1}{2}}-(\alpha^2+1)^{\frac{1}{2}}
\bigg)
\bigg)^{m_1}\,.
\end{equation}
And the proposition follows from Eq.~\eqref{eq: probability 44}.
\end{proof}

Now we discuss the choice of hyperparameters (\ie{} $m_1$ and $m_2$) according to Eq.~\eqref{eq: probability}. In fact, Eq.~\eqref{eq: probability} can be approximately written as $1-c\cdot n^{\frac{m_1}{k+1}}/m_2^{m_1}$. We can calculate the order of magnitude of $n^{\frac{m_1}{k+1}}/m_2^{m_1}$ by taking the logarithm: 
\begin{equation}\label{eq: magnitude}
\topequation\small
-\lambda\defineq \lg\Big(
n^{\frac{m_1}{k+1}} /m_2^{m_1}
\Big) =m_1 \Big(
\tfrac{\lg n}{k+1} -\lg m_2
\Big) \,.
\end{equation}
Therefore, \ppsalgabbr{} could reach an approximate solution with probability at least $(1-c\cdot 10^{-\lambda})$. In practice, we choose positive integers $m_2$ and $m_1$ such that $\lambda$ is reasonably large, ensuring that the algorithm will reach an approximate solution with high probability.

\section{Empirical Results}
In this section, we elaborate on our experiments to evaluate the effectiveness of the proposed \ppsdisabbr{} in Eq.~\eqref{eq:8} and \ppsuperabbr{} in Algorithm~\ref{alg:extend}, as well as \ppsalgabbr{} in Algorithm~\ref{alg:approx}. 
These experiments are conducted to explore the following research questions: 
\textbf{RQ1}. Compared with the state-of-the-art (SOTA) baseline fairness measures, does the proposed \ppsdisabbr{} capture the discriminative degree of one classifier effectively, and can it capture the discrimination level when facing several SAs with multiple values at the same time? 
Moreover, compared with the baselines, can \ppsdisabbr{} capture the discrimination level from both individual and group fairness aspects? 
\textbf{RQ2}. Can \ppsalgabbr{} approximate the direct computation of distances in Eq.~\eqref{eq:4a} and \eqref{eq:4b} precisely, and how efficient is \ppsalgabbr{} compared with the direct computation of distances? 
And by extension, can \ppsuperabbr{} approximate the direct computation of distances in Eq.~\eqref{eq:6a} and \eqref{eq:6b} precisely, and how efficient is \ppsuperabbr{} compared with the direct computation of distances? 
\textbf{RQ3}. Will the choice of hyperparameters (that is, $m_1$ and $m_2$ in \ppsalgabbr{} and \ppsuperabbr) affect the approximation results, and if the answer is yes, how? 
Furthermore, we also discuss the possibility of applying \ppsdisabbr{} to represented features in neural networks and the limitations of the proposed approximation methods at the end of this section.

\subsection{Experimental setup}
In this subsection, we present the experimental settings we use, including datasets, evaluation metrics, baseline fairness measures, and implementation details.

\paragraph*{\bf Datasets}
Five public datasets were adopted in the experiments; 
Each of them has two SAs except Ricci, with more details provided in Table~\ref{tab:stats} below.

\vspace{-3mm}
\begin{table}[tbh]%
\caption{Data statistics. The column `\#inst' represents the number of instances; 
`\#feat' represents the number of features (incl. one or two SAs) excluding classification labels, and `prep' is the number of features after preprocessing. 
For each SA, `\#val' means the number of its values; 
`\#in-priv' is the number of members in the privileged group accordingly. 
}\label{tab:stats}
\centering\vspace{-3mm}
\scalebox{.76}{%
\begin{tabular}{r|r|rr|rrr|rrr}
\toprule
\multirow{2}{*}{\bf Dataset} & \multirow{2}{*}{\bf \#inst} & \multicolumn{2}{c|}{\bf \#feat} & 
\multicolumn{3}{c|}{\bf 1st sensitive attribute} & \multicolumn{3}{c}{\bf 2nd sensitive attribute} \\
& & raw & prep & & \#val & \#in-priv & & \#val & \#in-priv \\
\midrule
ricci  \cite{dataset1_2024} &   118 &  5 &   6 & 
race& 3 & 68 & --- & --- & --- \\
credit \cite{dataset2_2024} &  1000 & 20 &  58 & 
sex & 2 & 690 & age & 2 & 851 \\
income \cite{dataset3_2024} & 30162 & 13 &  98 & 
race& 5 & 25933 & sex & 2 & 20380 \\
ppr    \cite{dataset4_2024} &  6167 & 10 & 401 & 
sex & 2 & 4994 & race & 6 & 2100 \\
ppvr   \cite{dataset4_2024} &  4010 & 10 & 327 & 
sex & 2 & 3173 & race & 6 & 1452 \\
\bottomrule
\end{tabular}
}
\end{table}

\paragraph*{\bf Evaluation metrics}
As data imbalance usually exists within unfair datasets, we consider several criteria to evaluate the prediction performance from different perspectives, including accuracy, precision, recall (\aka{} sensitivity), $\mathrm{f}_1$ score, and specificity. 
For efficiency metrics, we directly compare the time cost of different methods.

\paragraph*{\bf Baseline fairness measures}
To evaluate the validity of \ppsdisabbr{} in capturing the discriminative degree of classifiers, we compare it with 
three commonly-used group fairness measures (that is, demographic parity (DP) \cite{feldman2015certifying,gajane2017formalizing}, equality of opportunity (EO) \cite{hardt2016equality}, and predictive quality parity (PQP) \cite{chouldechova2017fair,verma2018fairness})%
\footnote{%
Three commonly used group fairness measures of one classifier $f(\cdot)$ are evaluated on only one bi-valued SA as
\begin{subequations}
\scriptsize\topequation
\begin{align}
    \mathrm{DP}(f) &= \lvert
    \mathbb{P}_\mathcal{D}[ f(\bm{x})\!=\!1 | \xpos\!=\!1 ] -
    \mathbb{P}_\mathcal{D}[ f(\bm{x})\!=\!1 | \xpos\!=\!0 ]
    \rvert \,,\\
    \mathrm{EO}(f) &= \lvert
    \mathbb{P}_\mathcal{D}[ f(\bm{x})\!=\!1 | \xpos\!=\!1,\, y\!=\!1 ] 
    - 
    \mathbb{P}_\mathcal{D}[ f(\bm{x})\!=\!1 | \xpos\!=\!0,\, y\!=\!1 ]
    \rvert ,\\ 
    \mathrm{PQP}(f) &= \lvert
    \mathbb{P}_\mathcal{D}[ y\!=\!1 | \xpos\!=\!1,\, f(\bm{x})\!=\!1 ] 
    - 
    \mathbb{P}_\mathcal{D}[ y\!=\!1 | \xpos\!=\!0,\, f(\bm{x})\!=\!1 ]
    \rvert ,
\end{align}%
\label{eq:10}%
\end{subequations}%
respectively, where $\insx= (\xneg,\xpos)$, 
$y$, and $f(\bm{x})$ are respectively the features, the true label, and the prediction of this classifier for one instance. Note that $\xpos=1$ and $0$ respectively mean that the instance $\insx$ belongs to the privileged group and the marginalised groups. 
}, %
one extended fairness measure (called statistical parity (SP), inspired by \cite{agarwal2019fair,jiang2020wasserstein})\footnote{%
The original statistical parity (SP) in \cite{agarwal2019fair,jiang2020wasserstein} is neither a quantitative measure nor applied directly to multiple SAs. It is more like, for at least one (bi- or multi-valued) SA, that is, $\xpos\!=\! [a_1,...,a_{n_a}]^\mathsf{T} (n_a\!\geqslant\! 1)$, 
\begin{subequations}
\scriptsize\topequation
\begin{align}
    \mathrm{SP}^\text{pre,max}(f) &= \textstyle
    \max_{k\in \mathcal{A}^\text{pre}} \{
    |\mathbb{P}_\mathcal{D}[ f(\insx)\!=\!1 \mid \xpos\!=\!k ]
    -\mathbb{P}_\mathcal{D}[ f(\insx)\!=\!1 ]|
    \} \,,\\
    \mathrm{SP}^\text{pre,avg}(f) &=\textstyle
    \sum_{k\in \mathcal{A}^\text{pre}}
    |\mathbb{P}_\mathcal{D}[ f(\insx)\!=\!1 \mid \xpos\!=\!k ]
    -\mathbb{P}_\mathcal{D}[ f(\insx)\!=\!1 ]| \,,
\end{align}%
\end{subequations}%
where $\xpos\in\mathcal{A}^\text{pre}$ and $\mathcal{A}^\text{pre}$ is a finite set with $\prod_{1\leqslant i\leqslant n_a}|n_{a_i}|$ elements. 
It is not exactly the same but consistent with \cite{agarwal2019fair,jiang2020wasserstein}.

For comparison, we made slight modifications following its original idea. 
For one multi-valued SA (that is, $a_i\!\in\! \{1,2,...,n_{a_i}\}$), two versions of the extended SP are 
\begin{subequations}
\scriptsize\topequation
\begin{align}
    \mathrm{SP}_i^\text{max}(f) &= \textstyle
    \max_{1\leqslant j\leqslant n_{a_i}}\{
        |\mathbb{P}_\mathcal{D}[ f(\insx)\!=\!1 \!\mid\! a_i\!=\!j ] 
        -\mathbb{P}_\mathcal{D}[ f(\insx)\!=\!1 ]|
    \} \,,\\
    \mathrm{SP}_i^\text{avg}(f) &= \textstyle
    \sum_{1\leqslant j\leqslant n_{a_i}}
        |\mathbb{P}_\mathcal{D}[ f(\insx)\!=\!1 \!\mid\! a_i\!=\!j ] 
        -\mathbb{P}_\mathcal{D}[ f(\insx)\!=\!1 ]|
    \,,
\end{align}%
\label{eq:fair,sp}%
\end{subequations}%
respectively. 
For several multi-valued SAs (that is, $\xpos\!=\! [a_1,...,a_{n_a}]^\mathsf{T}$), 
two versions of the extended SP are 
\begin{subequations}
\scriptsize\topequation
\begin{align}
    \mathrm{SP}^\text{max}(f) &= \textstyle
    \max_{1\leqslant i\leqslant n_a}
    \mathrm{SP}_i^\text{max}(f) \,,\\
    \mathrm{SP}^\text{avg}(f) &= \textstyle
    \frac{1}{n_a} \sum_{1\leqslant i\leqslant n_a}
    \mathrm{SP}_i^\text{avg}(f) \,,
\end{align}%
\label{eq:fair,dp}%
\end{subequations}%
respectively. 
Neither \eqref{eq:fair,sp} nor \eqref{eq:fair,dp} is exactly the same as the original SP.}, 
and one named discriminative risk (DR)\footnote{%
The discriminative risk (DR) of this classifier is evaluated as
\begin{equation}
\scriptsize\topequation
    \mathrm{DR}(f) = \mathbb{E}_\mathcal{D}[ 
        \mathbb{I}(f(\xneg,\xpos)\neq f(\xneg,\xqtb)) 
    ] \,,\label{eq:11}
\end{equation}%
where $\xqtb$ represents the perturbed SAs. DR reflects its bias level from both individual- and group-fairness aspects. 
} \cite{bian2023increasing_re} that could reflect the bias level of ML models from both individual- and group-fairness aspects. 
Note that only the extended SP and DR can work for several multi-valued SAs.

\begin{table*}
\begin{minipage}{\textwidth}
\centering
\caption{Test evaluation performance of different fairness measures, where LightGBM is used as the learning algorithm. The column named `$\text{Att}_\text{sen}$' denotes a corresponding SA, and the notation $\Delta$(performance) denotes the performance difference between a metric and that after perturbing the data \cite{bian2023increasing_re}. Note that $\newfist_\text{prev}=\sfrac{\extBistAlt{f}}{\extBistAlt{}}-1$ represents our previous work \cite{bian2024does}, and 
$\newfist=\log(\sfrac{ \newDist_{f,\xpos}(S,a_i) }{ \newDist_{\xpos}(S,a_i) })$ and $\newfist^\text{avg}=\log(\sfrac{ \newDist_{f,\xpos}^\text{avg}(S,a_i) }{ \newDist_{\xpos}^\text{avg}(S,a_i) })$ here represent \ppsdisabbr{} in this paper for each SA. 
}\label{tab:sen_att,sing}
\vspace{-4mm}%
\scalebox{.74}{%
\renewcommand\tabcolsep{2pt}
\begin{tabular}{rr| rrrr| rrrr r|rrr}
\toprule
\multirow{2}{*}{\bf Dataset} & \multirow{2}{*}{\bf $\text{Att}_\text{sen}$} & \multicolumn{4}{c|}{\bf Normal evaluation metric} & \multicolumn{5}{c|}{\bf Baseline fairness measure} & \multicolumn{3}{c}{\bf Proposed fairness measure} \\
\cline{3-14} 
& & \multicolumn{1}{c}{Accuracy} & \multicolumn{1}{c}{$\mathrm{f}_1$ score} & \multicolumn{1}{c}{$\Delta$Accuracy} & \multicolumn{1}{c|}{$\Delta\mathrm{f}_1$ score}
& \multicolumn{1}{c}{DP} & \multicolumn{1}{c}{EO} & \multicolumn{1}{c}{PQP} & \multicolumn{1}{c}{DR} & \multicolumn{1}{c|}{$\newfist_\text{prev}$ bin-val} 
& \multicolumn{1}{c}{$\newfist_\text{prev}$ multival} & \multicolumn{1}{c}{$\newfist$} & \multicolumn{1}{c}{$\newfist^\text{avg}$} \\
\midrule
ricci  & race &
99.5789\topelement{$\pm$0.5766} &	99.5604\topelement{$\pm$0.6019} &	52.2105\topelement{$\pm$0.5766} &	35.2747\topelement{$\pm$0.6019} &	0.3112\topelement{$\pm$0.0424} &	0.0000\topelement{$\pm$0.0000} &	0.0121\topelement{$\pm$0.0166} &	0.5221\topelement{$\pm$0.0058} &	0.0000\topelement{$\pm$0.0000} &	0.0000\topelement{$\pm$0.0000} &	0.0000\topelement{$\pm$0.0000} &	0.0016\topelement{$\pm$0.0022} \\
credit & sex  &
77.8750\topelement{$\pm$1.1726} &	86.2892\topelement{$\pm$0.6221} &	10.2750\topelement{$\pm$3.9906} &	11.7147\topelement{$\pm$4.7568} &	0.0189\topelement{$\pm$0.0095} &	0.0016\topelement{$\pm$0.0006} &	0.0666\topelement{$\pm$0.0189} &	0.3438\topelement{$\pm$0.1001} &	-0.0059\topelement{$\pm$0.0181} &	-0.0059\topelement{$\pm$0.0181} &	-0.0026\topelement{$\pm$0.0079} &	-0.0075\topelement{$\pm$0.0005} \\
       & age  &
77.8750\topelement{$\pm$1.1726} &	86.2892\topelement{$\pm$0.6221} &	10.2750\topelement{$\pm$3.9906} &	11.7147\topelement{$\pm$4.7568} &	0.0335\topelement{$\pm$0.0137} &	0.0065\topelement{$\pm$0.0037} &	0.1107\topelement{$\pm$0.0209} &	0.3438\topelement{$\pm$0.1001} &	-0.0047\topelement{$\pm$0.0105} &	-0.0047\topelement{$\pm$0.0105} &	-0.0021\topelement{$\pm$0.0046} &	-0.0073\topelement{$\pm$0.0008} \\
income & race &
83.3998\topelement{$\pm$0.2568} &	51.6536\topelement{$\pm$1.4002} &	3.8515\topelement{$\pm$3.6332} &	6.6956\topelement{$\pm$3.6031} &	0.0395\topelement{$\pm$0.0013} &	0.0126\topelement{$\pm$0.0050} &	0.0110\topelement{$\pm$0.0069} &	0.1542\topelement{$\pm$0.1015} &	-0.0414\topelement{$\pm$0.0218} &	-0.0414\topelement{$\pm$0.0218} &	-0.0185\topelement{$\pm$0.0099} &	-0.0170\topelement{$\pm$0.0012} \\
       & sex  &
83.3998\topelement{$\pm$0.2568} &	51.6536\topelement{$\pm$1.4002} &	3.8515\topelement{$\pm$3.6332} &	6.6956\topelement{$\pm$3.6031} &	0.0886\topelement{$\pm$0.0033} &	0.0793\topelement{$\pm$0.0089} &	0.0106\topelement{$\pm$0.0063} &	0.1542\topelement{$\pm$0.1015} &	-0.0075\topelement{$\pm$0.0160} &	-0.0075\topelement{$\pm$0.0160} &	-0.0033\topelement{$\pm$0.0071} &	-0.0073\topelement{$\pm$0.0007} \\
ppr    & sex  &
70.0507\topelement{$\pm$0.4676} &	62.9810\topelement{$\pm$1.4929} &	10.0709\topelement{$\pm$0.3289} &	1.4437\topelement{$\pm$0.9277} &	0.1861\topelement{$\pm$0.0207} &	0.1800\topelement{$\pm$0.0357} &	0.0169\topelement{$\pm$0.0082} &	0.3598\topelement{$\pm$0.0100} &	-0.0040\topelement{$\pm$0.0078} &	-0.0040\topelement{$\pm$0.0078} &	-0.0017\topelement{$\pm$0.0034} &	0.0051\topelement{$\pm$0.0103} \\
       & race &
70.0507\topelement{$\pm$0.4676} &	62.9810\topelement{$\pm$1.4929} &	10.0709\topelement{$\pm$0.3289} &	1.4437\topelement{$\pm$0.9277} &	0.1891\topelement{$\pm$0.0272} &	0.2192\topelement{$\pm$0.0297} &	0.0377\topelement{$\pm$0.0143} &	0.3598\topelement{$\pm$0.0100} &	-0.0134\topelement{$\pm$0.0134} &	-0.0114\topelement{$\pm$0.0104} &	-0.0050\topelement{$\pm$0.0046} &	-0.0154\topelement{$\pm$0.0139} \\
ppvr   & sex  &
83.8953\topelement{$\pm$0.2315} &	1.9415\topelement{$\pm$2.7688} &	0.1620\topelement{$\pm$0.2315} &	1.9415\topelement{$\pm$2.7688} &	0.0020\topelement{$\pm$0.0029} &	0.0113\topelement{$\pm$0.0162} &	0.4000\topelement{$\pm$0.5477} &	0.0016\topelement{$\pm$0.0023} &	-0.0107\topelement{$\pm$0.0625} &	-0.0107\topelement{$\pm$0.0625} &	-0.0054\topelement{$\pm$0.0268} &	-0.0560\topelement{$\pm$0.0042} \\
       & race &
83.8953\topelement{$\pm$0.2315} &	1.9415\topelement{$\pm$2.7688} &	0.1620\topelement{$\pm$0.2315} &	1.9415\topelement{$\pm$2.7688} &	0.0008\topelement{$\pm$0.0011} &	0.0048\topelement{$\pm$0.0093} &	0.0000\topelement{$\pm$0.0000} &	0.0016\topelement{$\pm$0.0023} &	-0.0150\topelement{$\pm$0.0930} &	-0.0027\topelement{$\pm$0.0253} &	-0.0013\topelement{$\pm$0.0109} &	-0.0785\topelement{$\pm$0.0036} \\
\bottomrule
\end{tabular}
}
\end{minipage}
\vspace{-3mm}
\begin{minipage}{\textwidth}
\centering
\caption{Test evaluation performance of different fairness measures, where LightGBM is used as the learning algorithm. The notation $\Delta$ denotes the performance difference between a metric and that after perturbing the data. Here $\text{DR}_i$ works for one SA with multiple values \cite{bian2023increasing_re}, and we use $\text{DR}_\text{avg}= \frac{1}{n_a}\sum_{i=1}^{n_a} \text{DR}_i$ to reflect the bias level on the whole dataset. 
}\label{tab:sen_att,pl}
\vspace{-4mm}%
\scalebox{.72}{%
\renewcommand\tabcolsep{1.5pt}
\begin{tabular}{r| rrrr| rrr| rrr| rrr}
\toprule
\multirow{2}{*}{\bf Dataset} & \multicolumn{4}{c|}{\bf Normal evaluation metric} & \multicolumn{3}{c|}{\bf Fairness for first sensitive attribute} & \multicolumn{3}{c|}{\bf Fairness for second sensitive attribute} & \multicolumn{3}{c}{\bf Fairness for all sensitive attributes} \\
\cline{2-14} 
& \multicolumn{1}{c}{Accuracy} & \multicolumn{1}{c}{$\mathrm{f}_1$ score} & \multicolumn{1}{c}{$\Delta$Accuracy} & \multicolumn{1}{c|}{$\Delta\mathrm{f}_1$ score}
& \multicolumn{1}{c}{DR$_1$} & \multicolumn{1}{c}{$\newfist_1$} & \multicolumn{1}{c|}{$\newfist_1^\text{avg}$}
& \multicolumn{1}{c}{DR$_2$} & \multicolumn{1}{c}{$\newfist_2$} & \multicolumn{1}{c|}{$\newfist_2^\text{avg}$}
& \multicolumn{1}{c}{DR$_\text{avg}$} & \multicolumn{1}{c}{$\newfist$} & \multicolumn{1}{c}{$\newfist^\text{avg}$} \\
\midrule
ricci  &
97.3913\topelement{$\pm$2.3814} &	97.3085\topelement{$\pm$2.4628} &	49.5652\topelement{$\pm$2.3814} &	32.6026\topelement{$\pm$2.4628} &	0.5130\topelement{$\pm$0.0364} &	0.0000\topelement{$\pm$0.0000} &	-0.0031\topelement{$\pm$0.0271} &	--- &	--- &	--- &	0.5130\topelement{$\pm$0.0364} &	0.0000\topelement{$\pm$0.0000} &	-0.0031\topelement{$\pm$0.0271} \\
credit &
77.8750\topelement{$\pm$1.1726} &	86.2892\topelement{$\pm$0.6221} &	10.2750\topelement{$\pm$3.9906} &	11.7147\topelement{$\pm$4.7568} &	0.3438\topelement{$\pm$0.1001} &	-0.0026\topelement{$\pm$0.0079} &	-0.0075\topelement{$\pm$0.0005} &	0.3438\topelement{$\pm$0.1001} &	-0.0021\topelement{$\pm$0.0046} &	-0.0073\topelement{$\pm$0.0008} &	0.3438\topelement{$\pm$0.1001} &	-0.0021\topelement{$\pm$0.0046} &	-0.0074\topelement{$\pm$0.0005} \\
income &
83.3998\topelement{$\pm$0.2568} &	51.6536\topelement{$\pm$1.4002} &	3.8515\topelement{$\pm$3.6332} &	6.6956\topelement{$\pm$3.6031} &	0.1542\topelement{$\pm$0.1015} &	-0.0185\topelement{$\pm$0.0099} &	-0.0170\topelement{$\pm$0.0012} &	0.1542\topelement{$\pm$0.1015} &	-0.0033\topelement{$\pm$0.0071} &	-0.0073\topelement{$\pm$0.0007} &	0.1542\topelement{$\pm$0.1015} &	-0.0041\topelement{$\pm$0.0068} &	-0.0107\topelement{$\pm$0.0005} \\
ppr    &
70.0507\topelement{$\pm$0.4676} &	62.9810\topelement{$\pm$1.4929} &	10.0709\topelement{$\pm$0.3289} &	1.4437\topelement{$\pm$0.9277} &	0.3598\topelement{$\pm$0.0100} &	-0.0017\topelement{$\pm$0.0034} &	0.0051\topelement{$\pm$0.0103} &	0.3598\topelement{$\pm$0.0100} &	-0.0050\topelement{$\pm$0.0046} &	-0.0154\topelement{$\pm$0.0139} &	0.3598\topelement{$\pm$0.0100} &	-0.0017\topelement{$\pm$0.0034} &	-0.0026\topelement{$\pm$0.0108} \\
ppvr   &
83.8953\topelement{$\pm$0.2315} &	1.9415\topelement{$\pm$2.7688} &	0.1620\topelement{$\pm$0.2315} &	1.9415\topelement{$\pm$2.7688} &	0.0016\topelement{$\pm$0.0023} &	-0.0054\topelement{$\pm$0.0268} &	-0.0560\topelement{$\pm$0.0042} &	0.0016\topelement{$\pm$0.0023} &	-0.0013\topelement{$\pm$0.0109} &	-0.0785\topelement{$\pm$0.0036} &	0.0016\topelement{$\pm$0.0023} &	-0.0054\topelement{$\pm$0.0268} &	-0.0647\topelement{$\pm$0.0034} \\
\bottomrule
\end{tabular}
}
\end{minipage}
\end{table*}
\begin{figure}
\centering %
\subfloat[]{\includegraphics[width=5.4cm]{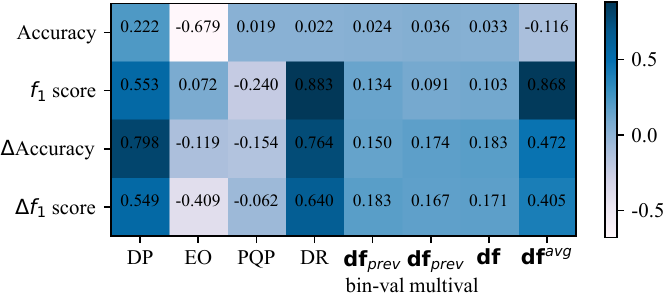}}\\
\vspace{-4mm}%
\subfloat[]{\includegraphics[width=5.4cm]{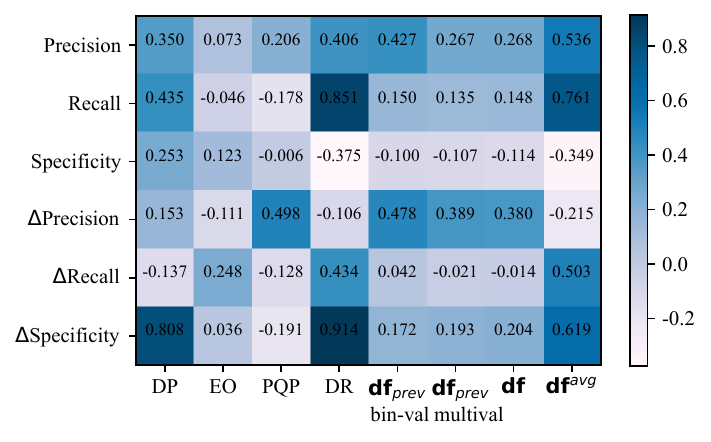}}
\vspace{-3mm}%
\caption{%
Correlation heatmap between normal evaluation metric and fairness, for one single SA. The used notations refer to those in Table~\ref{tab:sen_att,sing}. 
}\label{fig:sen_att,sing}%
\end{figure}
\begin{figure}
\centering %
\subfloat[]{\includegraphics[width=6.1cm]{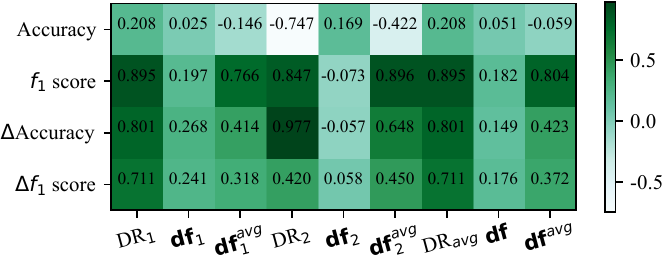}}\\
\vspace{-4mm}
\subfloat[]{\includegraphics[width=6.1cm]{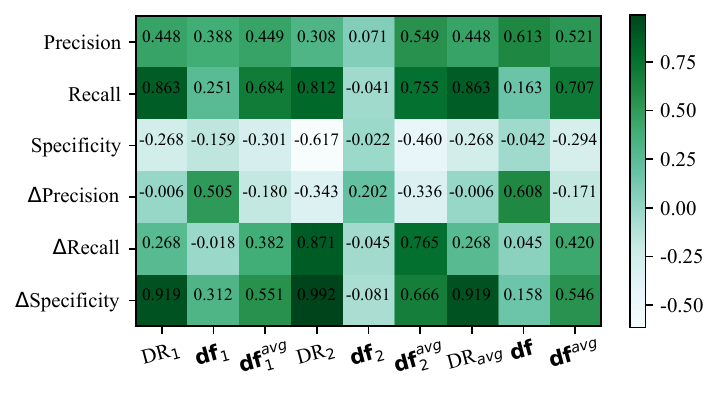}}
\vspace{-3mm}
\caption{%
Correlation heatmap between normal evaluation metric and fairness measure, for all SAs within the dataset. The notations used here refer to those in Table~\ref{tab:sen_att,pl}. 
}\label{fig:sen_att,pl}
\end{figure}

\begin{figure}
\centering
\subfloat[]{\label{subfig:extSP,a}
\includegraphics[height=2.8cm]{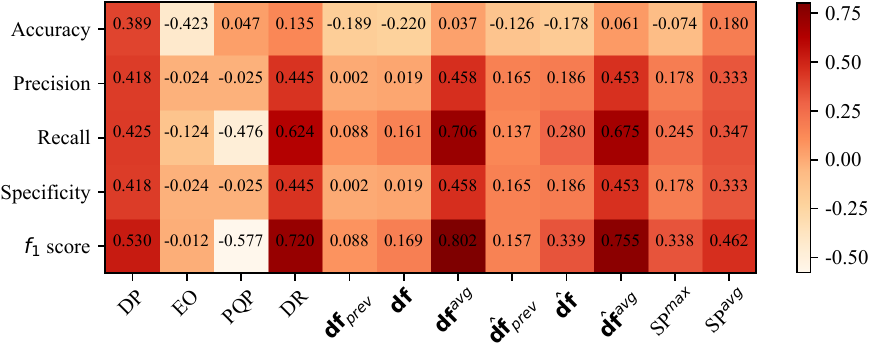}} \\ \vspace{-3mm}
\subfloat[]{\label{subfig:extSP,b}
\includegraphics[height=2.8cm]{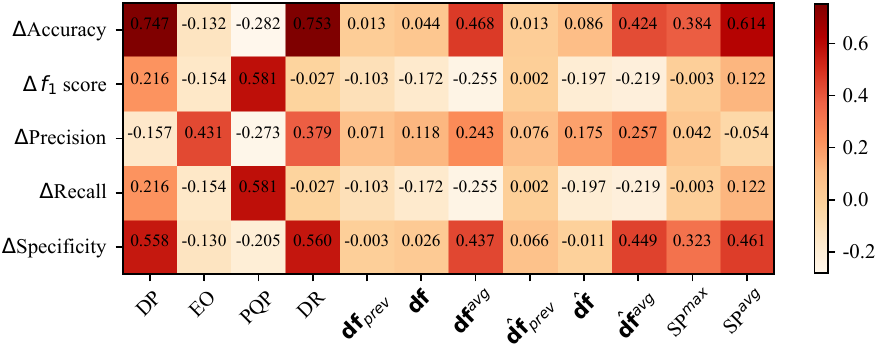}} \\ \vspace{-3mm}
\subfloat[]{\label{subfig:extSP,c}
\includegraphics[height=3.3cm]{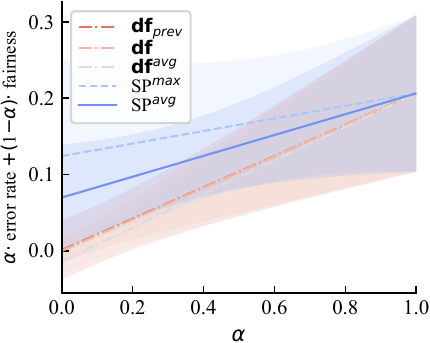}}
\subfloat[]{\label{subfig:extSP,d}
\includegraphics[height=3.3cm]{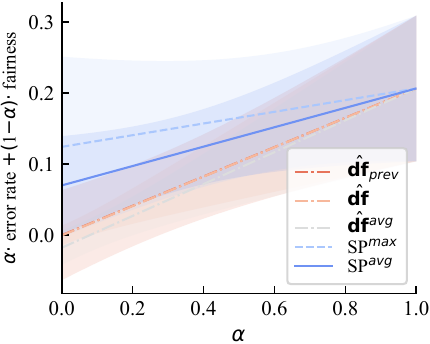}} \vspace{-3mm}
\caption{Comparison of \ppsdisabbr{} with the extended SP. 
(a--b) Correlation heatmap between normal evaluation metric and fairness measure, for all SAs within the dataset. Note that in the computation of $\newfist_\text{prev}$ here, multi-valued SAs are handled as bi-valued cases, equivalent to $\newfist_\text{prev\;bin-val}$ in Figure~\ref{fig:sen_att,sing}; $\hat{\newfist}_\text{prev}, \hat{\newfist}$, and $\hat{\newfist}^\text{avg}$ indicate that the distances are obtained using approximation algorithms. 
(c--d) Plots of best test-set fairness-accuracy trade-offs per fairness metric \cite{cruz2022fairgbm} (the smaller the better). 
}\label{fig:ext,SP}
\end{figure}

\begin{figure*}
\begin{minipage}{\textwidth}
\centering
\subfloat[]{\label{subfig:tradeoff,a}
\includegraphics[height=3.27cm]{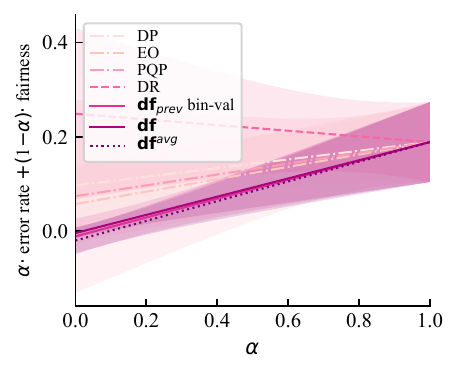}}
\subfloat[]{\label{subfig:tradeoff,b}
\includegraphics[height=3.27cm]{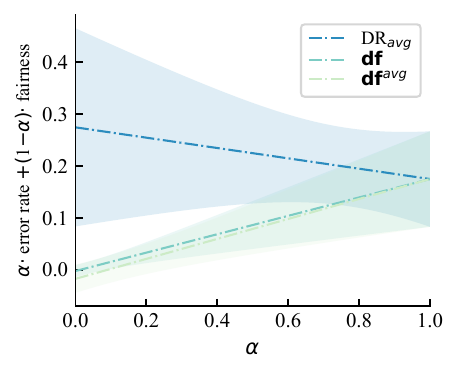}}
\subfloat[]{\label{subfig:tradeoff,c}
\includegraphics[height=3.27cm]{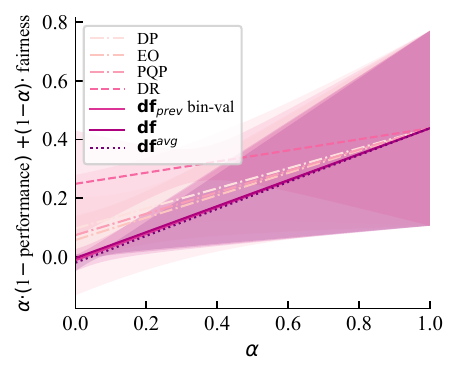}}
\subfloat[]{\label{subfig:tradeoff,d}
\includegraphics[height=3.27cm]{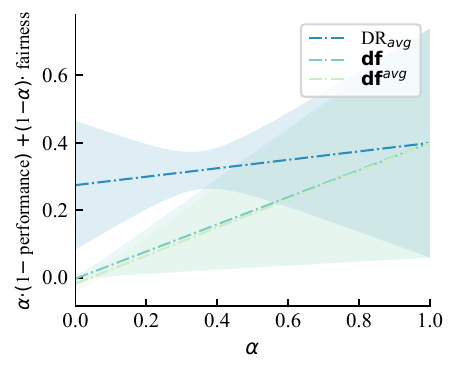}}
\vspace{-3mm}\caption{%
Plots of best test-set fairness-performance trade-offs per fairness metric \cite{cruz2022fairgbm} (the smaller the better). (a) Plot of fairness-accuracy trade-off for one single SA; (b) Plot of fairness-accuracy trade-off for all SAs; 
(c--d) Plots of fairness-$\mathrm{f}_1$ score trade-off for one SA and for all SAs, respectively. 
Note that the notations in (a) and (c) refer to those in Table~\ref{tab:sen_att,sing}, and that in (b) and (d) refer to those in Table~\ref{tab:sen_att,pl}.  
}\label{fig:trade-off}
\end{minipage}
\vspace{-4mm}
\begin{minipage}{\textwidth}
\centering
\subfloat[]{\label{subfig:approx,1}
\includegraphics[width=4.1cm]{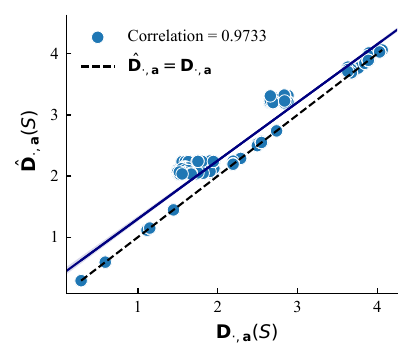}}
\subfloat[]{\label{subfig:approx,2}
\includegraphics[width=4.1cm]{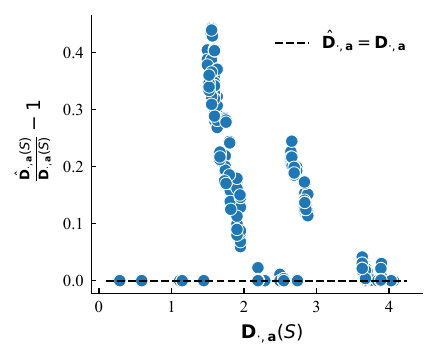}}
\subfloat[]{\label{subfig:approx,3}
\includegraphics[width=4.1cm]{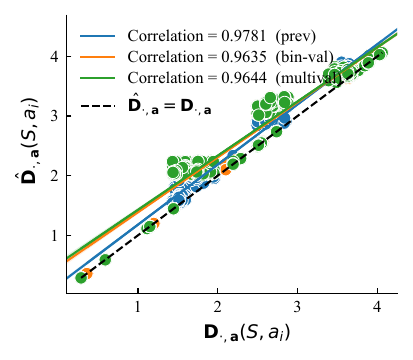}}
\subfloat[]{\label{subfig:approx,4}
\includegraphics[width=4.1cm]{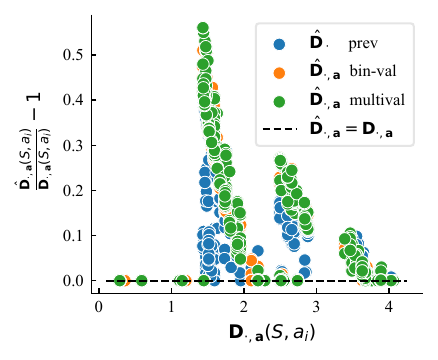}}
\\ \vspace{-4mm}
\subfloat[]{\label{subfig:approx,5}
\includegraphics[width=4.1cm]{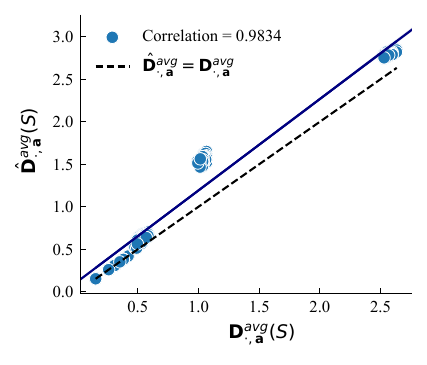}}
\subfloat[]{\label{subfig:approx,6}
\includegraphics[width=4.1cm]{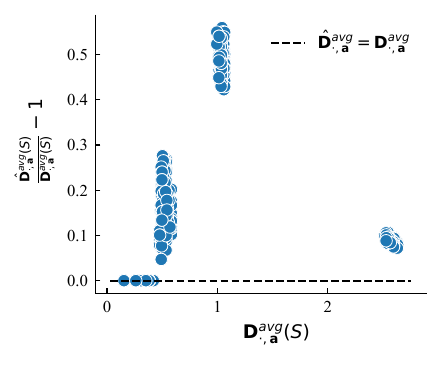}}
\subfloat[]{\label{subfig:approx,7}
\includegraphics[width=4.1cm]{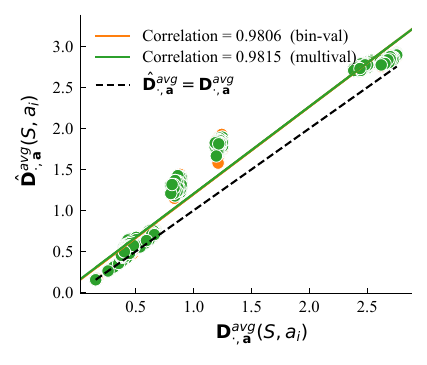}}
\subfloat[]{\label{subfig:approx,8}
\includegraphics[width=4.1cm]{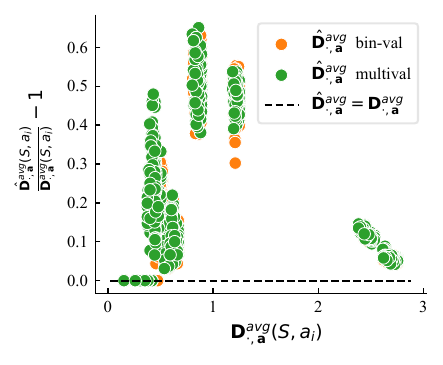}}
\\ \vspace{-4mm}
\subfloat[]{\label{subfig:approx,9}
\includegraphics[width=4.1cm]{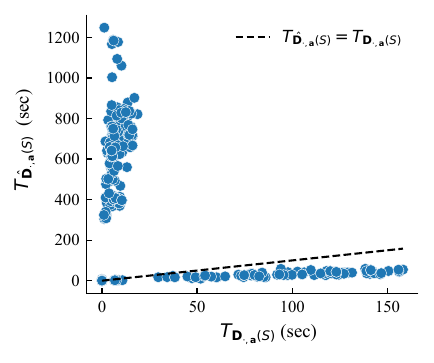}}
\subfloat[]{\label{subfig:approx,10}
\includegraphics[width=4.1cm]{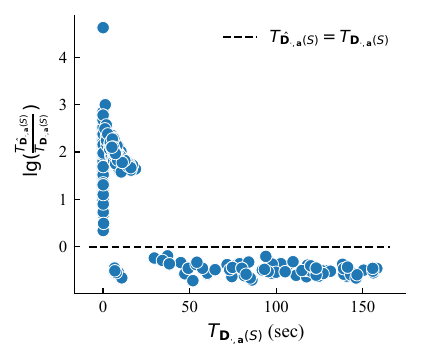}}
\subfloat[]{\label{subfig:approx,11}
\includegraphics[width=4.1cm]{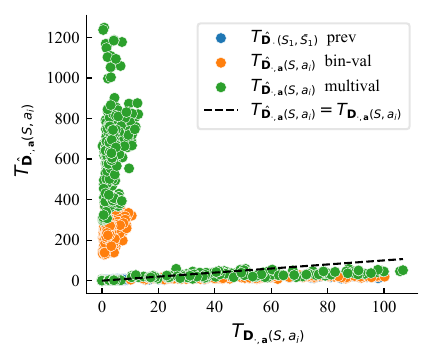}}
\subfloat[]{\label{subfig:approx,12}
\includegraphics[width=4.1cm]{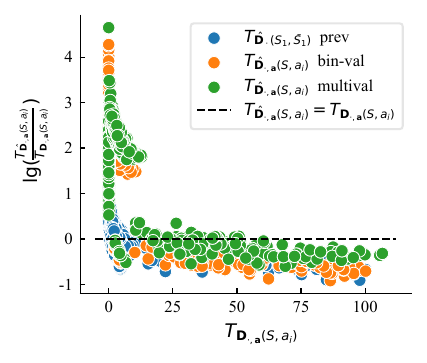}}
\vspace{-3mm}
\caption{Comparison of approximation distances between sets with precise distances that are calculated directly by definition, evaluated on test data; Note that `prev' denotes the approximation results obtained by our previous work \cite{bian2024does}. 
(a--b), (c--d), (e--f), and (g--h) Scatter plots for comparison between approximated and precise values of $\extEist{}$, $\extDist{}$, $\extEist{avg}$, and $\extDist{avg}$, respectively; 
(i--j) Time cost comparison between \ppsuperabbr{} and direct computation via Eq.~\eqref{eq:6a} and \eqref{eq:6b}; (k--l) Time cost comparison between \ppsalgabbr{} and direct computation via Eq.~\eqref{eq:4a} and \eqref{eq:4b}. 
}\label{fig:approx}
\end{minipage}
\end{figure*}

\begin{figure*}
\begin{minipage}{\textwidth}
\centering
\subfloat[]{\label{subfig:pm,ext,1}
\includegraphics[height=3.4cm]{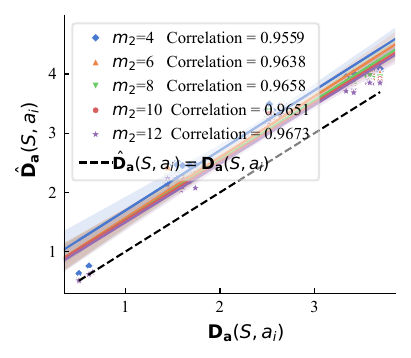}}
\subfloat[]{\label{subfig:pm,ext,2}
\includegraphics[height=3.4cm]{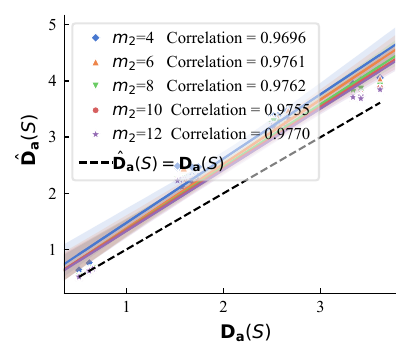}}
\subfloat[]{\label{subfig:pm,ext,3}
\includegraphics[height=3.4cm]{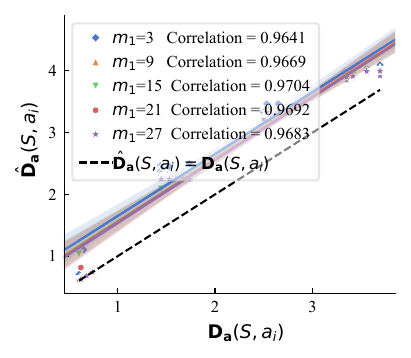}}
\subfloat[]{\label{subfig:pm,ext,4}
\includegraphics[height=3.4cm]{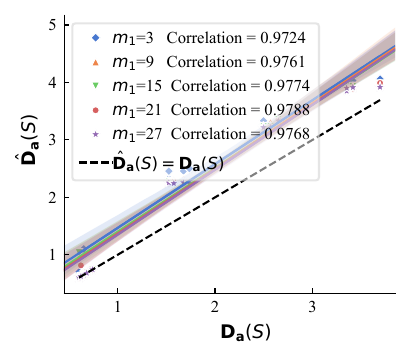}}
\\ \vspace{-4mm}
\subfloat[]{\label{subfig:pm,ext,5}
\includegraphics[height=3.4cm]{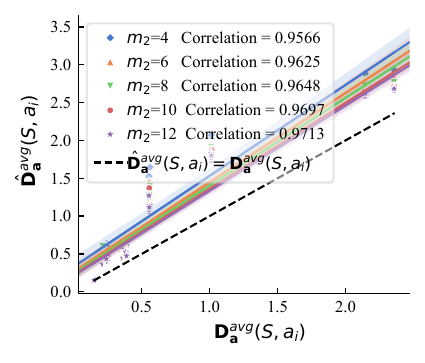}}
\subfloat[]{\label{subfig:pm,ext,6}
\includegraphics[height=3.4cm]{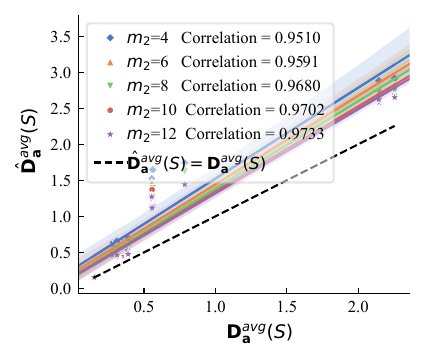}}
\subfloat[]{\label{subfig:pm,ext,7}
\includegraphics[height=3.4cm]{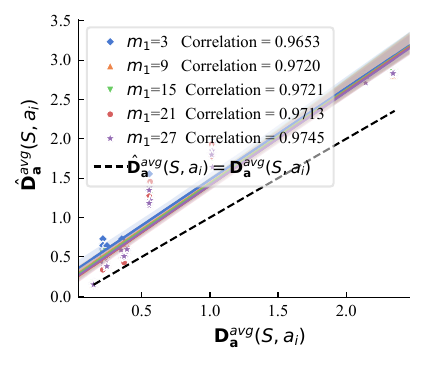}}
\subfloat[]{\label{subfig:pm,ext,8}
\includegraphics[height=3.4cm]{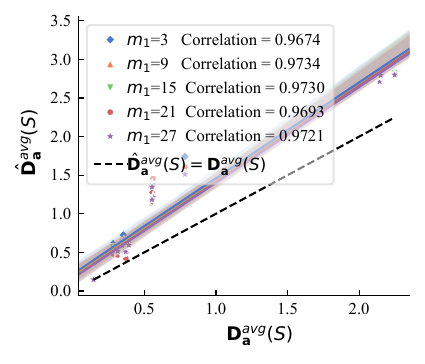}}
\\ \vspace{-4mm}
\subfloat[]{\label{subfig:pm,ext,9}
\includegraphics[height=3.4cm]{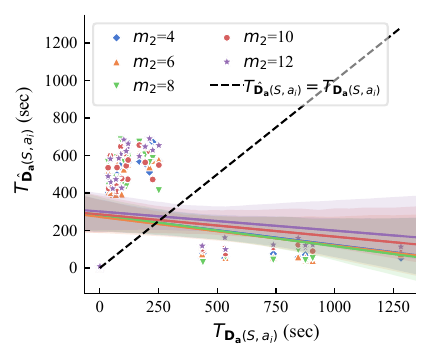}}
\subfloat[]{\label{subfig:pm,ext,10}
\includegraphics[height=3.4cm]{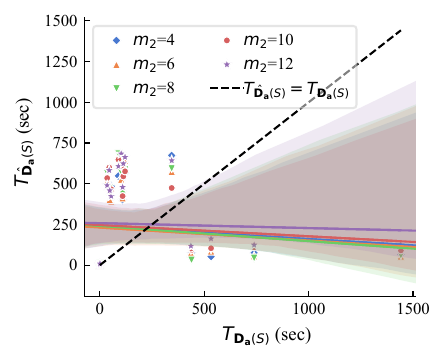}}
\subfloat[]{\label{subfig:pm,ext,11}
\includegraphics[height=3.4cm]{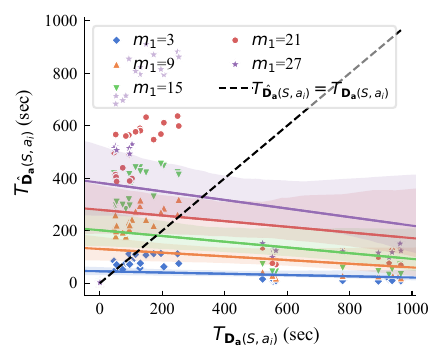}}
\subfloat[]{\label{subfig:pm,ext,12}
\includegraphics[height=3.4cm]{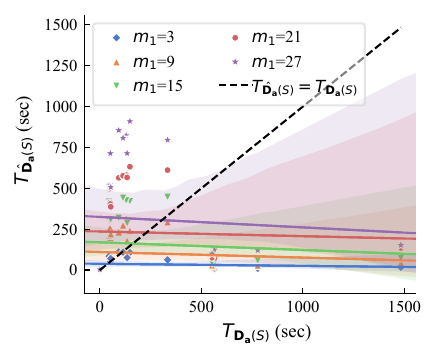}}
\vspace{-3mm}
\caption{Effect of hyperparameters $m_1$ and $m_2$ in \ppsalgabbr{} and \ppsuperabbr{}. 
(a--b) The effect of $m_2$ on the maximal distance value; (c--d) The effect of $m_1$ on the maximal distance values. 
(e--f) The effect of $m_2$ on the average distance value; (g--h) The effect of $m_1$ on the average distance value. 
(i--j) The effect of $m_2$ on the time cost; (k--l) The effect of $m_1$ on the time cost, where $m_2$ is set to $\lceil 2\lg(n) \rceil$ in terms of $n$---the size of the corresponding dataset. 
}\label{fig:hyper,ext}
\end{minipage}
\begin{minipage}{\textwidth}
\centering
\subfloat[]{\label{subfig:pm,bin,1}
\includegraphics[height=3.4cm]{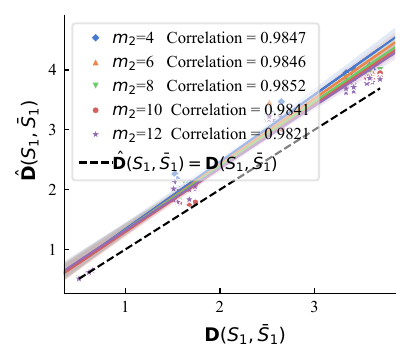}}
\subfloat[]{\label{subfig:pm,bin,2}
\includegraphics[height=3.4cm]{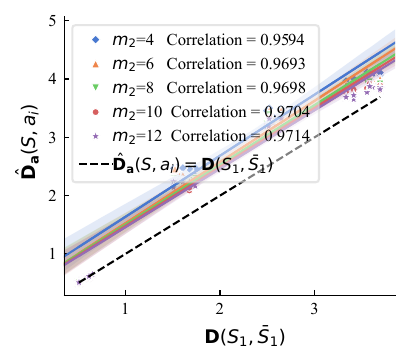}}
\subfloat[]{\label{subfig:pm,bin,3}
\includegraphics[height=3.4cm]{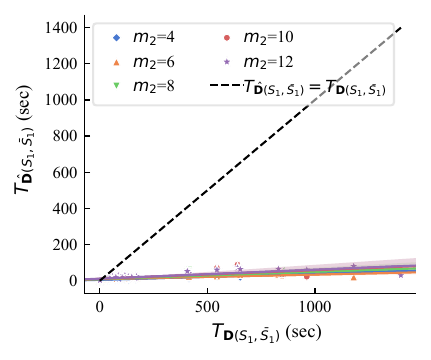}}
\subfloat[]{\label{subfig:pm,bin,4}
\includegraphics[height=3.4cm]{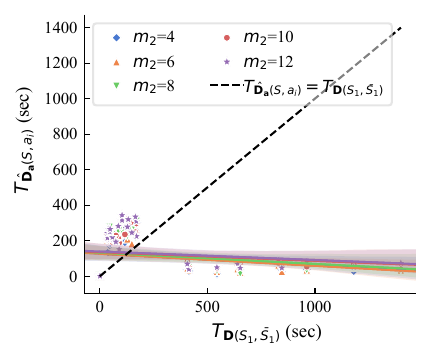}}
\\ \vspace{-4mm}
\subfloat[]{\label{subfig:pm,bin,5}
\includegraphics[height=3.4cm]{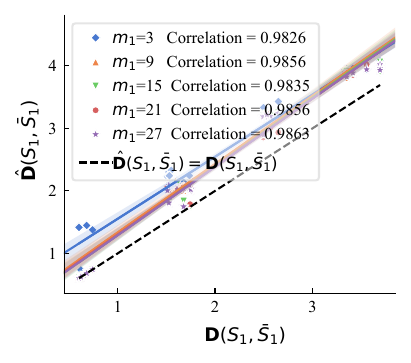}}
\subfloat[]{\label{subfig:pm,bin,6}
\includegraphics[height=3.4cm]{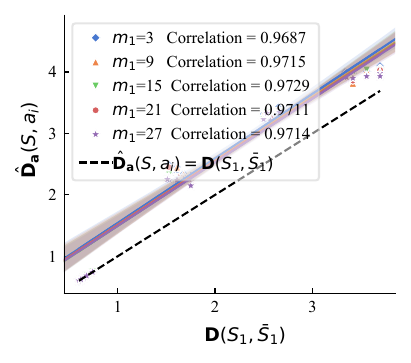}}
\subfloat[]{\label{subfig:pm,bin,7}
\includegraphics[height=3.4cm]{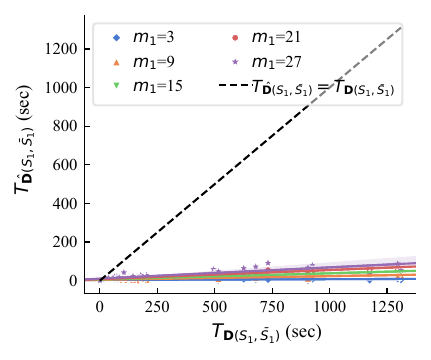}}
\subfloat[]{\label{subfig:pm,bin,8}
\includegraphics[height=3.4cm]{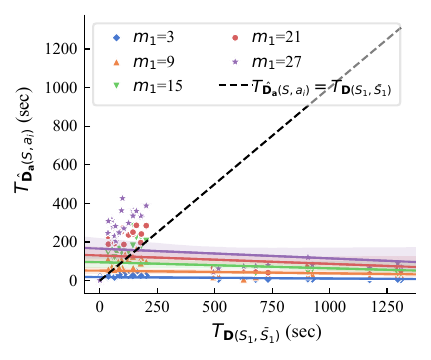}}
\vspace{-3mm}
\caption{Effect of hyperparameters $m_1$ and $m_2$ in our previous work \cite{bian2024does} and \ppsalgabbr{} here, where only binary values are considered for SAs. 
(a--b) The effect of $m_2$ on the (maximal) distance value; (c--d) The effect of $m_2$ on the time cost. 
(e--f) The effect of $m_1$ on the (maximal) distance value; (g--h) The effect of $m_1$ on the time cost, where $m_2$ is fixed to $\lceil 2\lg(n) \rceil$ in terms of $n$---the size of the corresponding dataset. 
}\label{fig:hyper,bin}
\end{minipage}
\end{figure*}

\paragraph*{\bf Implementation details}
We mainly use bagging, AdaBoost, LightGBM \cite{ke2017lightgbm}, FairGBM \cite{cruz2022fairgbm}, and AdaFair \cite{iosifidis2019adafair} as learning algorithms, where FairGBM and AdaFair are two fairness-aware ensemble-based methods. 
Plus, certain kinds of classifiers are used in Section~\ref{subsec:RQ1}---including decision trees (DT), naive Bayesian (NB) classifiers, $k$-nearest neighbours (KNN) classifiers, Logistic Regression (LR), support vector machines (SVM), linear SVMs (linSVM), and multilayer perceptrons (MLP)---so that we have a larger learner pool to choose from based on different fairness-relevant rules. 
Standard 5-fold cross-validation is used in these experiments; in other words, in each iteration, the entire dataset is divided into two parts, with 80\% as the training set and 20\% as the test set. 
Also, features of datasets are scaled in preprocessing to lie between 0 and 1. 
Except for the experiments for RQ3, we set the hyperparameters $m_1=25$ and $m_2= \lceil 2\lg(n)\rceil$ in all other experiments. 
The value $25$ of $m_1$ is randomly chosen, just as a value that is not too small to get an awful approximation meanwhile not too large to cost unnecessarily much more time. We also discuss, in RQ3 and Section~\ref{subsec:RQ3}, whether the choice of $m_1$ and $m_2$ will affect the approximation results.

\subsection{Comparison between \ppsdisabbr{} and baseline fairness measures}
\label{subsec:RQ1}
The aim of this experiment is to evaluate 
the effectiveness of the proposed \ppsdisabbr{} compared with baseline fairness measures. 
As ground-truth discriminative levels of classifiers remain unknown and it is hard to directly compare different methods from that perspective, we compare the correlation (referring to the Pearson correlation coefficient) between the performance difference and different fairness measures. 
The empirical results are reported in Fig.~\ref{fig:sen_att,sing}--\ref{fig:trade-off} and Tables~\ref{tab:sen_att,sing}--\ref{tab:sen_att,pl}.

\subsubsection{Comparison with three group fairness measures}
For one single SA, we can see from Fig.~\ref{fig:sen_att,sing} that: even $\newfist^\text{avg}$ only describes the extra bias, its correlation with $\Delta($performance$)$ is still close to that of DR (and sometimes DP), which means \ppsdisabbr{} can capture the bias within classifiers indeed and that \ppsdisabbr{} captures it more finely than our previous work \cite{bian2024does}. 
Moreover, $\newfist^\text{avg}$ shows higher correlation with $\Delta($performance$)$ than $\newfist$ in most cases, which means $\newfist^\text{avg}$ may capture the extra bias level of classifiers better than $\newfist$ in practice. 

As for multiple SAs, we can see from Fig.~\ref{fig:sen_att,pl} that $\newfist^\text{avg}$ shows higher correlation with $\Delta($performance$)$ than $\newfist$ in most cases, which is similar to our observation in Fig.~\ref{fig:sen_att,sing}. 
Note that the original DR \cite{bian2023increasing_re} calculates all SAs with binary or multiple values as a whole, 
and for comparison with \ppsdisabbr{}, we calculate here $\text{DR}_i$ for each SA and $\text{DR}_\text{avg}= \frac{1}{n_a}\sum_{i=1}^{n_a} \text{DR}_i$, analogously to $\newfist^\text{avg}$. 
Besides, we observe that the correlation between $\newfist^\text{avg}$ and $\Delta\text{Accuracy}$ (resp. $\Delta\mathrm{f}_1$ score, $\Delta\text{Specificity}$) achieves half of that of DR, and $\newfist^\text{avg}$ even outperforms DR concerning $\Delta\text{Recall}$. 
Given that \ppsdisabbr{} only captures the extra bias introduced by classifiers, we believe that at least $\newfist^\text{avg}$ could capture quite a part of the bias within.

Furthermore, we report plots of fairness-performance trade-offs per fairness measure in Fig.~\ref{fig:trade-off}. We can see that: 
1) for one single SA, \ppsdisabbr{} (\ie{} $\newfist$ and $\newfist^\text{avg}$) achieves the best result in Fig.~\ref{fig:trade-off}\subref{subfig:tradeoff,a} and \ref{fig:trade-off}\subref{subfig:tradeoff,c}; and 
2) for all SAs on one dataset, $\newfist$ and $\newfist^\text{avg}$ perform closely and both outperform $\text{DR}_\text{avg}$ in Fig.~\ref{fig:trade-off}\subref{subfig:tradeoff,b} and \ref{fig:trade-off}\subref{subfig:tradeoff,d}. 
This observation demonstrates the effectiveness of \ppsdisabbr{} from another perspective, in other words, \ppsdisabbr{} could work well if fairness-performance trade-offs need to be considered. 

\subsubsection{Comparison with the extended SP}

In Fig.~\ref{fig:ext,SP}\subref{subfig:extSP,b}, $\newfist^\text{avg}$ and $\hat{\newfist}^\text{avg}$ achieve higher correlation with $\Delta$Accuracy, $\Delta$Precision, and $\Delta$Specificity, compared with the extended $\mathrm{SP}^\text{max}$. 
Although they do not have as high a correlation as the extended $\mathrm{SP}^\text{avg}$, their correlations are close to that of the extended $\mathrm{SP}^\text{avg}$, meaning that \ppsdisabbr{} can capture the discrimination at least partially, consistent with the extra discrimination that it is supposed to capture. 
Besides, Fig.~\ref{fig:ext,SP}\subref{subfig:extSP,c} and \ref{fig:ext,SP}\subref{subfig:extSP,d} show that \ppsdisabbr{} (\ie{} both $\newfist$ and $\newfist^\text{avg}$, as well as their approximated values) achieves better fairness-accuracy trade-offs than the extended SP, which is beneficial.\looseness=-1

\subsection{Validity of approximation algorithms for distances between sets in Euclidean spaces}
\label{subsec:RQ2}
In this subsection, we evaluate the performance of the proposed \ppsalgabbr{} and \ppsuperabbr{} compared with the precise distance that is directly calculated by definitions. 
To verify whether they could achieve the true distance between sets in a precise and timely manner, we employ scatter plots to compare their values and time cost, presented in Fig.~\ref{fig:approx}. 
Note that $\extDistAlt{}$ and $\extDistAlt{avg}$ are computed together in \ppsalgabbr{} at one time, and so are $\extEistAlt{}$ and $\extDistAlt{avg}$ in \ppsuperabbr{}. 
Notice that the previous \ppsalgabbr{} \cite{bian2024does} is included for comparison to its current version in scenarios of binary values. 
For comparison, we also use one more baseline named EarlyBreak \cite{taha2015efficient}, which is an efficient algorithm for calculating the exact Hausdorff distance (HD), that is, the distance of sets in this paper.

\subsubsection{Validity of \ppsalgabbr}
As we can see from Figures~\ref{fig:approx}\subref{subfig:approx,3} and \ref{fig:approx}\subref{subfig:approx,4}, the approximated values of maximal distance $\extDistAlt{}$ are highly correlated with their corresponding precise values. Besides, their linear fit line and the identity line (that is, $f(x)=x$) are near and almost parallel, which means the approximated values are pretty close to their precise value. 
Similar observations are concluded for the average distance $\extDistAlt{avg}$ shown in Figures~\ref{fig:approx}\subref{subfig:approx,7} and \ref{fig:approx}\subref{subfig:approx,8}. 
As for the execution time of approximation and direct computation in Figures~\ref{fig:approx}\subref{subfig:approx,11} and \ref{fig:approx}\subref{subfig:approx,12}, \ppsalgabbr{} may take a bit longer time in scenarios of multi-value cases than that of binary values, while all of them could achieve a shorter time than precise values when the execution of direct computation is costly.

\subsubsection{Validity of \ppsuperabbr}
As we can see from Figures~\ref{fig:approx}\subref{subfig:approx,1} and \ref{fig:approx}\subref{subfig:approx,2}, 
the approximated values of maximal distance $\extEistAlt{}$ are highly correlated with their corresponding precise values. 
Besides, their linear fit line and the identity line are near and almost parallel, which means the approximated values are pretty close to their precise value. 
Similar observations are concluded for the average distance $\extEistAlt{avg}$ shown in Figures~\ref{fig:approx}\subref{subfig:approx,5} and \ref{fig:approx}\subref{subfig:approx,6}. 
As for the execution time of approximation and direct computation in Figures~\ref{fig:approx}\subref{subfig:approx,9} and \ref{fig:approx}\subref{subfig:approx,10}, \ppsuperabbr{} would obtain a bigger advantage when computing precise values is expensive, while on the opposite, we do not need \ppsuperabbr{} that much and can directly calculate them instead.

\subsubsection{Comparison with EarlyBreak \cite{taha2015efficient}}

We use EarlyBreak \cite{taha2015efficient} to compare with \ppsalgabbr{}, in order to evaluate whether \ppsalgabbr{} can efficiently approximate the exact values of distances. 
The reason why we do not compare it with \ppsuperabbr{} is that EarlyBreak and the original exact Hausdorff distance are only for single-SA cases. 
The empirical results are provided in Fig.~\ref{fig:efficienthd}. 
We can observe that: (1) Both EarlyBreak and \ppsalgabbr{} can achieve close computational or approximated values to the exact values of distances, with shorter time; (2) \ppsalgabbr{} usually requires less time cost than EarlyBreak, shown in Fig.~\ref{fig:efficienthd}\subref{subfig:efficient,b}. 
Thus, \ppsalgabbr{} can at least serve as an acceptable solution even though $\mathcal{O}(n\log n)$ may still pose an absolute computational cost challenge as $n$ increases.

\begin{figure*}
\centering
\subfloat[]{\includegraphics[height=3.3cm]{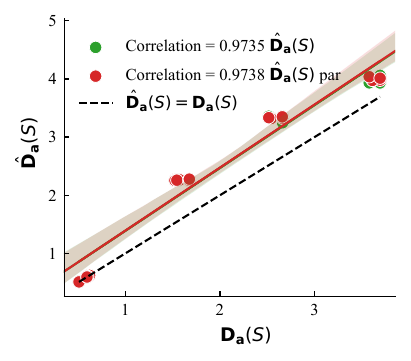}}
\subfloat[]{\includegraphics[height=3.3cm]{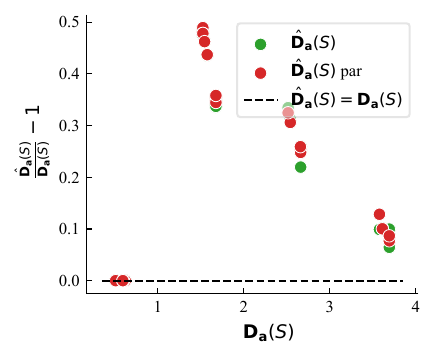}}
\subfloat[]{\includegraphics[height=3.3cm]{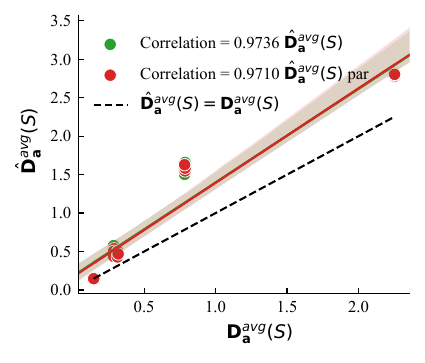}}
\subfloat[]{\includegraphics[height=3.3cm]{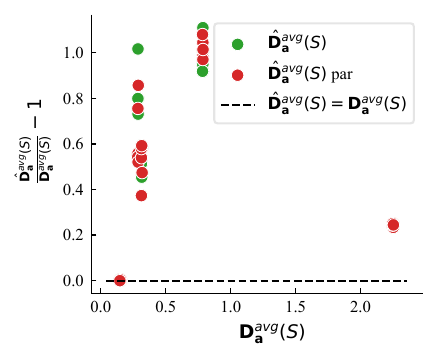}}
\\ \vspace{-4mm}
\subfloat[]{\includegraphics[height=3.3cm]{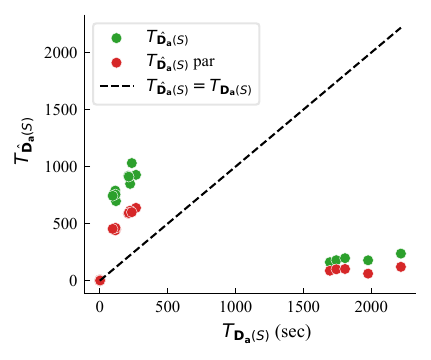}}
\subfloat[]{\includegraphics[height=3.3cm]{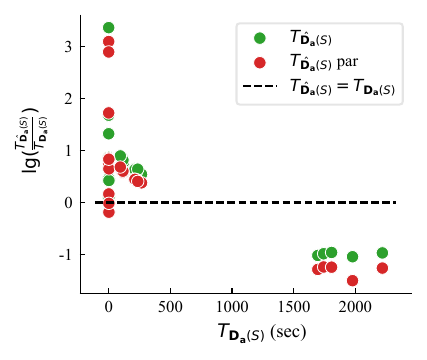}}
\subfloat[\label{subfig:par,sp}]{
\includegraphics[height=3.3cm]{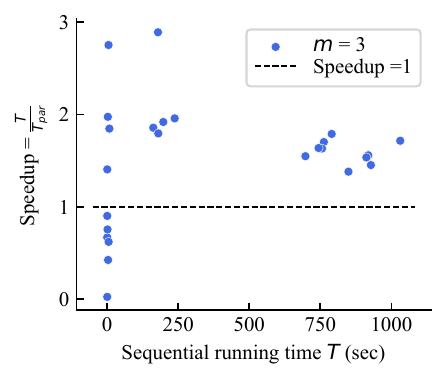}}
\subfloat[\label{subfig:par,ep}]{
\includegraphics[height=3.3cm]{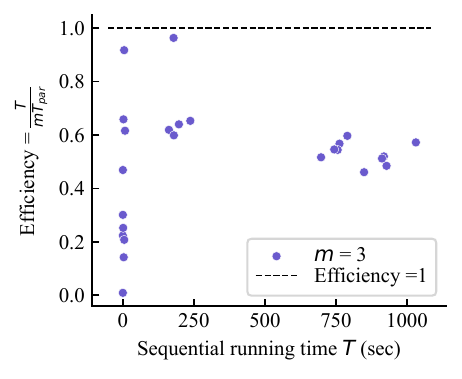}}
\vspace{-3mm}%
\caption{Comparison of approximation distances between sets with or without integrating parallel computing, where `par' represents results integrating parallel computing and three cores ($m=3$) are used in practice. 
(a--b) Scatter plots for comparison between approximated values and precise values of $\mathbf{D}_{\xpos}(S)$; (c--d) Scatter plots for comparison between approximated values and precise values of $\mathbf{D}_{\xpos}^\text{avg}(S)$; 
(e--f) Scatter plots for time cost comparison between obtaining approximation values and obtaining precise values of $\mathbf{D}_{\xpos}(S)$ and $\mathbf{D}_{\xpos}^\text{avg}(S)$; 
(g--h) Speedup and efficiency for comparison between with and without integrating parallel computing. Note that integrating parallel computing does accelerate the execution of \ppsuperabbr{} much without sacrificing approximation performance. 
}\label{fig:par_comput}
\end{figure*}

\begin{figure}
\centering
\subfloat[]{\includegraphics[height=3.5cm]{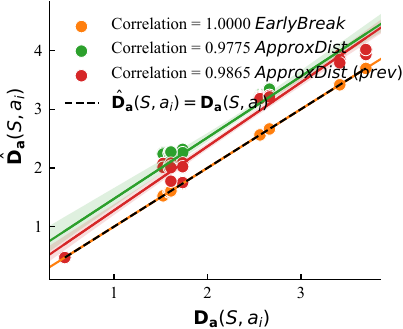}}
\subfloat[]{\label{subfig:efficient,b}
\includegraphics[height=3.5cm]{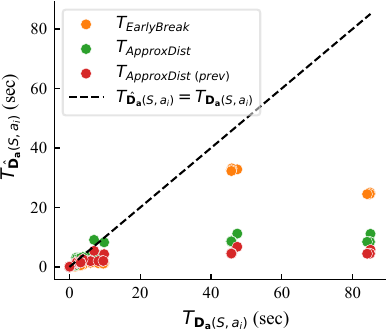}}
\caption{Comparison between \ppsalgabbr{} and EarlyBreak \cite{taha2015efficient}. Note that \ppsalgabbr{} (prev) is the approximation algorithm from our previous work \cite{bian2024does}, which is also supposed to approximate $\mathbf{D}_\cdot(S_1,\bar{S}_1)$, equivalent to $\mathbf{D}_{\cdot,\xpos}(S,a_i)$ in this paper.}
\label{fig:efficienthd}
\end{figure}

\subsection{Effect of hyperparameters $m_1$ and $m_2$}
\label{subsec:RQ3}
In this subsection, we investigate whether different choices of hyperparameters (that is, $m_1$ and $m_2$) would affect the performance of \ppsalgabbr{} and \ppsuperabbr{} or not. 
Different $m_2$ values are tested when $m_1$ is fixed, and vice versa, with empirical results presented in Figures~\ref{fig:hyper,ext} and \ref{fig:hyper,bin}. 

\subsubsection{Effect on \ppsalgabbr}
As we can see from Figures~\ref{fig:hyper,ext}\subref{subfig:pm,ext,9} and \ref{fig:hyper,ext}\subref{subfig:pm,ext,11}, 
when direct computation of distances (\ie{} maximal distance $\extDistAlt{}$ and average distance $\extDistAlt{avg}$) is expensive, obtaining their approximated values via \ppsalgabbr{} distinctly costs less time than that of precise values by Eq.~\eqref{eq:4a} and \eqref{eq:4b}. 
Increasing $m_2$ (or $m_1$) in \ppsalgabbr{} would cost more time, while the effect of increasing $m_1$ is more obvious. 

As for the approximation performance of $\extDistAlt{}$ shown in Fig.~\ref{fig:hyper,ext}\subref{subfig:pm,ext,1} and \ref{fig:hyper,ext}\subref{subfig:pm,ext,3} as well as approximation performance of $\extDistAlt{avg}$ shown in Fig.~\ref{fig:hyper,ext}\subref{subfig:pm,ext,5} and \ref{fig:hyper,ext}\subref{subfig:pm,ext,7}, 
all approximated values are highly correlated and close to the precise values of distance, no matter how small $m_2$ (or $m_1$) is, which means the effect of improper choices of hyperparameters is unapparent; 
As $m_2$ increases, the approximated values would be closer to the precise values of distance, while the effect of changing $m_1$ would be less manifest.

\subsubsection{Effect on \ppsuperabbr}
As we can see from Figures~\ref{fig:hyper,ext}\subref{subfig:pm,ext,10} and \ref{fig:hyper,ext}\subref{subfig:pm,ext,12}, 
when direct computation of distances (\ie{} maximal distance $\extEistAlt{}$ and average distance $\extEistAlt{avg}$) is expensive, 
obtaining their approximated values via \ppsuperabbr{} distinctly costs less time than that of precise values by Eq.~\eqref{eq:6a} and \eqref{eq:6b}. 
Increasing $m_2$ (or $m_1$) in \ppsuperabbr{} would cost more time, while the effect of increasing $m_1$ is more obvious. 

As for the approximation performance of $\extEistAlt{}$ shown in Figures~\ref{fig:hyper,ext}\subref{subfig:pm,ext,2} and \ref{fig:hyper,ext}\subref{subfig:pm,ext,4} as well as approximation performance of $\extEistAlt{avg}$ shown in Figures~\ref{fig:hyper,ext}\subref{subfig:pm,ext,6} and \ref{fig:hyper,ext}\subref{subfig:pm,ext,8}, 
all approximated values are highly correlated and close to the precise values of distance no matter how small $m_2$ (or $m_1$) is, which means the effect of improper choices of hyperparameters is unapparent; 
As $m_2$ increases, the approximated values would be closer to the precise values of distance, while the effect of changing $m_1$ would be less manifest.

\subsubsection{Comparison of \ppsalgabbr{} between our previous work \cite{bian2024does} and the current version in this paper} 
Furthermore, we also present the comparison between our previous work \cite{bian2024does} and \ppsalgabbr{} (Algorithm~\ref{alg:approx}) in Fig.~\ref{fig:hyper,bin}. 

We can see from Figures~\ref{fig:hyper,bin}\subref{subfig:pm,bin,1}, \ref{fig:hyper,bin}\subref{subfig:pm,bin,2}, \ref{fig:hyper,bin}\subref{subfig:pm,bin,5}, and \ref{fig:hyper,bin}\subref{subfig:pm,bin,6} that our previous \ppsalgabbr{} demonstrates slightly higher correlation to precise values of maximal distance $\extDistAlt{}$ (also known as $\newDist(S_1,\bar{S}_1)$ in scenarios of binary values) than its current version in this work; 
Different choices of $m_2$ (or $m_1$) cause nearly imperceptible effects on their approximation effectiveness. 
The previous version also shows close and even better performance on compressed time cost than the current \ppsalgabbr{} here, 
depicted in Figures~\ref{fig:hyper,bin}\subref{subfig:pm,bin,3}, \ref{fig:hyper,bin}\subref{subfig:pm,bin,4}, \ref{fig:hyper,bin}\subref{subfig:pm,bin,7}, and \ref{fig:hyper,bin}\subref{subfig:pm,bin,8}, especially when direct computation is not much expensive. 
However, when the execution time cost of direct computation is relatively cheap, \ppsuperabbr{} displays a messy and worse execution speed than \ppsalgabbr{}, shown in Figures~\ref{fig:hyper,ext}\subref{subfig:pm,ext,11} and \ref{fig:hyper,ext}\subref{subfig:pm,ext,12}. 
We believe there are mainly two reasons for this phenomenon: one is that \ppssubabbr{} is repeated twice from line \ref{line:app,2} to line \ref{line:app,5} in Algorithm~\ref{alg:approx} while it is executed only once in the previous \ppsalgabbr{} \cite{bian2024does}, causing Algorithm~\ref{alg:approx} a slightly longer execution time than its previous version; the other is that parallel computing is integrated in \ppsuperabbr{} in practice to further accelerate its execution, detailed more in Section~\ref{subsec:expt,limit} and Fig.~\ref{fig:par_comput}.

\begin{figure}
\centering
\subfloat[]{\includegraphics[height=3.4cm]{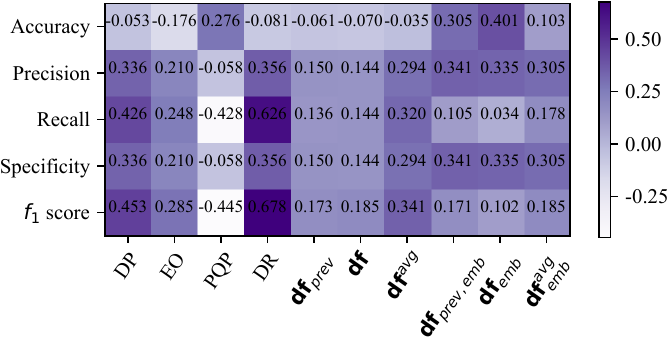}}\\ \vspace{-3mm}
\subfloat[]{\includegraphics[height=3.4cm]{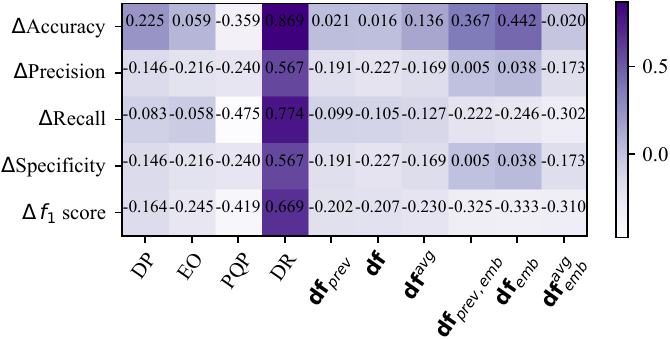}}\\ \vspace{-3mm}
\subfloat[]{\includegraphics[height=3.4cm]{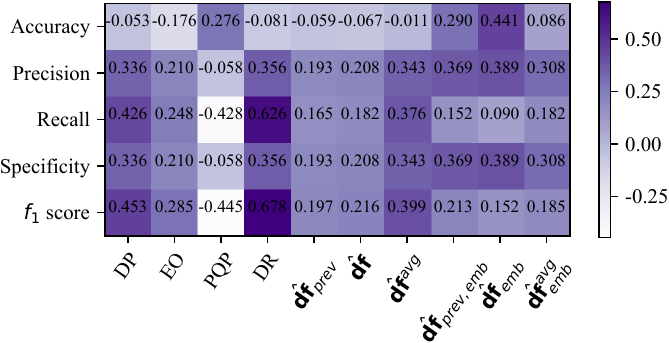}}\\ \vspace{-3mm}
\subfloat[]{\includegraphics[height=3.4cm]{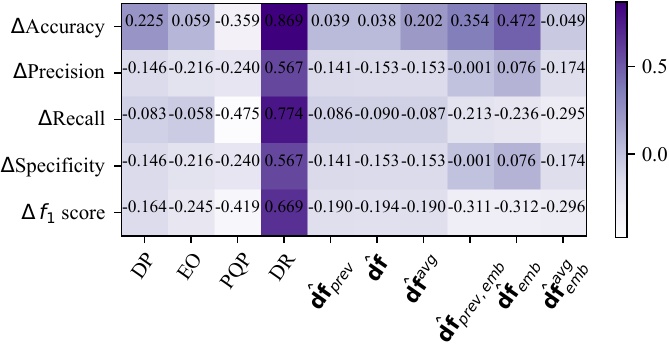}} \vspace{-3mm}
\caption{Potential exploration for more complex features, using a multilayer perceptron (MLP). 
Note that $\newfist_\text{prev}, \newfist, \newfist^\text{avg}$ are calculated based on original features, while $\newfist_\text{prev,emb}, \newfist_\text{emb}, \newfist_\text{emb}^\text{avg}$ are calculated using replaced features in the MLP, as described in Section~\ref{subsec:embed}. 
Also note that $\hat{\newfist}$ in (c) and (d) indicates that the distances are calculated using approximation algorithms, while $\newfist$ indicates the distances are calculated using the direct computation; So are the others. 
}\label{fig:embedding}
\end{figure}
\subsection{Potential exploration for more complex feature sets}
\label{subsec:embed}

We did  some preliminary exploration into the possibility of applying \ppsdisabbr{} to more complex feature sets, such as visual features. 
Instead of directly using \ppsdisabbr{} on visual datasets, we use the represented (or embedded) features before the final output of a neural network as an alternative way to test it. 
In other words, we still compute $\newDist{}_{\xpos}(S,a_i)$ with the original features, but replace those features with the represented vectors in neural networks in $\newDist{}_{f,\xpos}(S,a_i)$, and then calculate $\newDist{}_{\xpos}(S)$, $\newDist{}_{f,\xpos}(S)$, and $\newfist(f)$ as usual. 
Analogously, we can get $\newDist^\text{avg}_{f,\xpos}(S,a_i)$, $\newDist^\text{avg}_{f,\xpos}(S)$, and $\newfist^\text{avg}(f)$. 
The neural network we used in this experiment was a multilayer perceptron, where the input dimension depends on the dataset in use, 
followed by one hidden layer with a size of 256, ReLU, one hidden layer with a size of 64, ReLU, and one fully connected layer before Softmax. 
The empirical results are presented in Fig.~\ref{fig:embedding}, where we can observe that $\newfist_\text{emb}$ using the embedded features shows a higher correlation with $\Delta$Accuracy, compared with $\newfist$, meaning that \ppsdisabbr{} has the potential to be applied to complex features.

\subsection{Discussion and limitations}
\label{subsec:expt,limit}

Given the wide applications of ML models in the real world nowadays and the complexity of discrimination mitigation in the face of multiple factors interweaving, it matters a lot to bring in such techniques to deal with several SAs with even multiple values. 
Therefore, our work provides a fine-grained fairness measure option named \ppsdisabbr{} that captures the bias level of models more finely, in order to better detect and moderate discrimination within. 
The proposed \ppsdisabbr{} is suitable for both binary and multi-class classification, thus enlarging its applicable value. 
To efficiently approximate the value of \ppsdisabbr{}, we further proposed \ppsalgabbr{} and \ppsuperabbr{} to speed up the expensive calculation process, of which the effectiveness and efficiency have been demonstrated in Section~\ref{subsec:RQ2}. 
However, there are also limitations in the proposed approximation algorithms. 
The major one is that their time cost will significantly increase if the number of possible values within a SA is relatively large. 
For instance, the computation incurring on the PPR/PPVR datasets may take close or sometimes even longer time than that on the Income dataset, even though the latter has way more instances than the former, because there are six sub-groups under the race attribute on PPR/PPVR while the number is only five on Income. 
Therefore, we integrate parallel computing in practice to further raise the execution speed of \ppsuperabbr{}. 
To be specific, we use three cores to run 
lines~\ref{line:sup,1} to \ref{line:sup,3} of Algorithm~\ref{alg:extend} in parallel in our experiments (Fig.~\ref{fig:par_comput}), while the choice of the number of cores is not a fixed constant. In other words, using two or four cores is also acceptable if the practitioners like it. 
Furthermore, it is easy to tell that there is still room for improvement in the approximation algorithms. 
For instance, it might achieve faster computing times in \ppsalgabbr{} if the procedure between lines~\ref{line:app,1} and \ref{line:app,8} in Algorithm~\ref{alg:approx} could be executed in parallel, although we did not perform that this time. 
So does that between lines~\ref{line:acc,3} and \ref{line:acc,11} of Algorithm~\ref{alg:accele}. 
But we believe it will need a more deliberate design to balance the parallel computing among them in case the cost of generating more threads/processes to achieve it is not worth enough as expected compared with its computing results. 
Therefore, we would rather leave it for future work, instead of cramming too much and confusing our main contributions in this work.\looseness=-1

\section{Conclusion}

In this paper, we investigate how to evaluate the discrimination level of classifiers in the face of multi-attribute protection scenarios and present a novel harmonic fairness measure with two optional versions (that is, maximum \ppsdisabbr{} and average \ppsdisabbr{}), of which both are based on distances between sets from a manifold perspective. 
To accelerate the computation of distances between sets and reduce its time cost from $\comp(n^2)$ to $\comp(n\log n)$, we further propose two approximation algorithms (that is, \ppsalgabbr{} and \ppsuperabbr{}) to resolve bias evaluation in scenarios for single attribute protection and multi-attribute protection, respectively. 
Furthermore, we provide an algorithmic effectiveness analysis for \ppsalgabbr{} under certain assumptions to explain how well it could work theoretically. 
The empirical results have demonstrated that the proposed fairness measure (including maximum \ppsdisabbr{} and average \ppsdisabbr{}) and approximation algorithms (\ie{} \ppsalgabbr{} and \ppsuperabbr{}) are valid and effective.\looseness=-1

%
\bibliographystyle{IEEEtran}
\bibliography{nus_title_iso4,refsmac}

\end{document}